\documentclass{article}

\PassOptionsToPackage{numbers, compress}{natbib}

\usepackage[final]{nips_2018}

\usepackage[utf8]{inputenc} \usepackage[T1]{fontenc}    \usepackage{booktabs}       \usepackage{amsfonts}       \usepackage{nicefrac}       \usepackage{microtype}      \usepackage{amsmath}

\usepackage{caption}
\usepackage{amssymb}
\usepackage{mathtools}
\usepackage{bigints}
\usepackage{algorithm}
\usepackage{algorithmic}
\usepackage{xcolor}
\usepackage[capitalize]{cleveref}
\usepackage{mathrsfs}
\usepackage{amsthm}

\usepackage{multibib}
\newcites{app}{Supplementary References}

\usepackage[disable,backgroundcolor = White,textwidth=\marginparwidth]{todonotes}
\newcommand{\todog}[2][]{\todo[color=orange!25,size=\small,#1]{G: #2}}

\newcommand{\todor}[2][]{\todo[color=blue!25,size=\small,#1]{R: #2}}
\newtheorem{prop}{Proposition}
\newtheorem{definition}{Definition}
\newtheorem{theorem}{Theorem}
\newtheorem{lem}{Lemma}
\newtheorem{rmk}{Remark}

\Crefname{assumption}{Assumption}{Assumptions}
\crefname{theorem}{Theorem}{Theorems}
\crefname{lem}{Lemma}{Lemmas}
\crefname{corollary}{Corollary}{Corollaries}
\crefname{claim}{Claim}{Claims}
\crefname{prop}{Proposition}{Propositions}
\crefname{equation}{Equation}{Equations}
\crefname{rmk}{Remark}{Remarks}

\newcommand{\real}{\mathbb{R}}
\newcommand{\eps}{\epsilon}
\newcommand{\Vol}{\text{Vol}}
\newcommand{\ball}[2]{B_{#1}(#2)}
\newcommand{\cM}{\mathcal{M}}

\newcommand{\fprec}{\text{Precision}^f}

\newcommand{\rotate}{\mathcal{R}}
\newcommand{\cV}{\mathcal{V}}

\newcommand{\td}{\text{d}}
\newcommand{\Proj}{\text{Proj}}

\newcommand{\Prob}{\mathbb{P}}
\newcommand{\norm}[1]{\left\lVert#1\right\rVert}
\newcommand{\Lim}[1]{\raisebox{0.5ex}{\scalebox{0.8}{$\displaystyle \lim_{#1}\;$}}}

\def\presubsection{0in}

\def\postsubsection{0in}

\title{Dimensionality Reduction has Quantifiable Imperfections: Two Geometric Bounds}

\author{
                      Kry Yik Chau Lui \\
  Borealis AI\\
  Canada\\
  \texttt{yikchau.y.lui@borealisai.com}
  \And
  Gavin Weiguang Ding\\
  Borealis AI\\
  Canada\\
  \texttt{gavin.ding@borealisai.com}
  \AND
  Ruitong Huang \\
  Borealis AI\\
  Canada\\
  \texttt{ruitong.huang@borealisai.com}
  \And
  Robert J. McCann\\
  Department of Mathematics\\
  University of Toronto\\
  Canada\\
  \texttt{mccann@math.toronto.edu}
          }

\begin{document}

\maketitle

\begin{abstract}
					In this paper, 
	we investigate Dimensionality reduction (DR) maps in an information retrieval setting from a quantitative topology point of view.
	In particular, 
	we show that no DR maps can achieve perfect precision and perfect recall simultaneously. 
	Thus a continuous DR map must have imperfect precision.
	We further prove an upper bound on the precision of Lipschitz continuous DR maps.
	While precision is a natural measure in an information retrieval setting, 
	it does not measure `how' wrong the retrieved data is.
	We therefore propose a new measure based on Wasserstein distance that comes with similar theoretical guarantee.
	A key technical step in our proofs is a particular optimization problem of the $L_2$-Wasserstein distance over a constrained set of distributions.
	We provide a complete solution to this optimization problem, 
	which can be of independent interest on the technical side.
\end{abstract}

\section{Introduction}
\label{sec:intro}

Dimensionality reduction (DR) serves as a core problem in machine learning tasks
including information compression, clustering, manifold learning, feature extraction, logits and other modules in a neural network and data visualization \citep{Hjaltason:2003:PEM:776753.776797, boutsidis2010random, venna2010information, lecun2015deep, mcqueen2016nearly}.
In many machine learning applications, 
the data manifold is reduced to a dimension lower than its intrinsic dimension (e.g. for data visualizations, output dimension is reduced to 2 or 3; for classifications, it is the number of classes). 
In such cases, it is not possible to have a continuous bijective DR map (i.e. classic algebraic topology result on invariance of dimension \citep{muger2015remark}). 
With different motivations, 
many nonlinear DR maps have been proposed in the literature, 
such as Isomap, kernel PCA, and t-SNE, 
just to name a few \citep{scholkopf1997kernel,tenenbaum2000global,maaten2008visualizing}.
A common way to compare the performances of different DR maps is to use a down stream supervised learning task as the ground truth performance measure.
However, when such down stream task is unavailable, 
e.g.  in an unsupervised learning setting as above, 
one would have to design a performance measure based on the particular context.
In this paper, we focus on the information retrieval setting, which falls into this case. 
An information retrieval system extracts the features $f(x)$ from the raw data $x$ for future queries. 
When a new query $y_0 = f(x_0)$ is submitted, the system returns the most relevant data with similar features, i.e. all the $x$ such that $f(x)$ is close to $y_0$. 
For computational efficiency and storage, $f$ is usually a DR map, retaining only the most informative features. 
Assume that the ground truth relevant data of $x_0$ is defined as a neighbourhood $U$ of $x$ that is a ball with radius $r_U$ centered at $x$ 
\footnote{The value of $r_U$ is unknown, and it depends on the user and the input data $x_0$. However, we can assume $r_U$ is small compared to the input domain size. For example, the number of relevant items to a particular user is much fewer than the number of total items. }, 
and the system retrieves the data based on relevance in the feature space, 
i.e. 
the inverse image, $f^{-1}(V)$, of a retrieval neighbourhood $V \ni f(x_0)$. 
Here $V$ is the ball centered at $y_0 = f(x_0)$ with radius $r_V$ that is determined by the system. 
It is natural to measure the system's performance based on the discrepancy between $U$ and $f^{-1}(V)$. 
Many empirical measures of this discrepancy have been proposed in the literature,
among which precision and recall are arguably the most popular ones \citep{schreck2010techniques,martins2014visual,lespinats2011checkviz, venna2010information}.
However, theoretical understandings of these measures are still very limited.

In this paper, we start with analyzing the theoretical properties of precision and recall in the information retrieval setting. 
Naively computing precision and recall in the discrete settings gives undesirable properties,
e.g. 
precision always equals recall when computed by using $k$ nearest neighbors.
How to measure them properly is unclear in the literature (\cref{sec:simulation}).
On the other hand, numerous experiments have suggested that there exists a tradeoff between the two when dimensionality reduction happens \citep{venna2010information}, yet this tradeoff still remains a conceptual mystery in theory. 
To theoretically understand this tradeoff, 
we look for continuous analogues of precision and recall, and exploit the geometric and function analytic tools that study dimensionality reduction maps \citep{guth2012waist}. 
The first question we ask is what property a DR map should have, 
so that the information retrieval system can attain zero false positive error (or false negative error) when the relevant neighbourhood $U$ and the retrieved neighbourhood $V$ are properly selected.
Our analyses show the equivalence between the achievability of perfect recall (i.e. zero false negative) and the continuity of the DR map.
We further prove that no DR map can achieve both perfect precision and perfect recall simultaneously.
Although it may seem intuitive, 
to our best knowledge, 
this is the first theoretical guarantee in the literature of the necessity of the tradeoff between precision and recall in a dimension reduction setting.

Our main results are developed for the class of (Lipschitz) continuous DR maps.
The first main result of this paper is an upper bound for the precision of a continuous DR map. 
We show that given a continuous DR map, 
its precision decays exponentially fast with respect to the number of (intrinsic) dimensions reduced. 
To our best knowledge, 
this is the first theoretical result in the literature for the decay rate of the precision of a dimensionality reduction map. 
The second main result is an alternative measure for the performance of a continuous DR map, called $W_2$ measure, based on $L_2$-Wasserstein distance. 
This new measure is more desirable as it can also detect the distance distortion between $U$ and $f^{-1}(V)$. 
Moreover, we show that our measure also enjoys a theoretical lower bound for continuous DR maps.
Several other distance-based measures have been proposed in the literature \citep{schreck2010techniques,martins2014visual,lespinats2011checkviz, venna2010information}, yet all are proposed heuristically with meagre theoretical understanding. 
Simulation results suggest optimizing the Wasserstein measure lower bound corresponds to optimizing a weighted f-1 score (i.e. f-$\beta$ score).
Thus we can optimize precision and recall without dealing with their computational difficulties in the discrete settings.

Finally, 
let us make some comments on the technical parts of the paper. 
The first key step is the Waist Inequality from the field of quantitative algebraic topology. 
At a high level, 
we need to analyse $f^{-1}(V)$, 
inverse image of an open ball for an arbitrary continuous map $f$. 
The waist inequality guarantees the existence of a `large' fiber, 
which allows us to analyse $f^{-1}(V)$ and prove our first main result. 
We further show that in a common setting, 
a significant proportion of fibers are actually `large'.
For our second main result, 
a key step in the proof is a complete solution to the following iterated optimization problem: 
\[
 \inf_{W:\,\Vol_n(W) = M} W_{2}(\mathbb{P}_{B_r}, \mathbb{P}_{W}) 
 = 
 \inf_{W:\,\Vol_n(W) = M} \inf_{\gamma \in \Gamma (\mathbb{P}_{B_r} , \mathbb{P}_{W})} \mathbb{E}_{(a, b) \sim \gamma} [ \| a - b \|^{2}_{2} ]^{1/2} ,
\]
where $B_r$ is a ball with radius $r$, $\Prob_{B_r}$ ($\Prob_W$, respectively) is a uniform distribution over $B_r$ ($W$, respectively), 
and $W_2$ is the $L_2$-Wasserstein distance. 
Unlike a typical optimal transport problem where the transport function between source and target distributions is optimized, 
in the above problem the source distribution is also being optimized at the outer level. 
This becomes a difficult constrained iterated optimization problem. 
To address it, 
we borrow tools from optimal partial transport theory \citep{caffarelli2010free, figalli2010optimal}. 
Our proof techniques leverage the uniqueness of the solution to the optimal partial transport problem and the rotational symmetry of $B_r$ to deduce $W$. 

\subsection{Notations}
\label{subsec:notations}
We collect our notations in this section. Let $m$ be the embedding dimension,
$\mathcal{M}$ be an $n$ dimensional data manifold\footnote{There is empirical and theoretical evidence that data distribution lies on low dimensional submanifold in the ambient space \citep{narayanan2010sample}. }~embedded in $\mathbb{R}^N$,
where $N$ is the ambient dimension.
$\mathcal{M} $ is typically modelled as a Riemannian manifold, 
so it is a metric space with a volume form. 
Let $ m < n < N$ and $f: \cM \subset \mathbb{R}^N \to \mathbb{R}^m$ be a DR map.
The pair ($x$, $y$) will be the points of interest,
where $y = f(x)$.
The inverse image of $y$ under the map $f$ is called fiber,
denoted $f^{-1}(y)$.
We say $f$ is continuous at point $x$ iff $\text{osc}^{f}(x) = 0$, 
where $\text{osc}^{f}(x) = \inf_{U; U \text{open}}\{\text{diam}( f(U) ); x \in U\}$ is the oscillation for $f$ at $x \in \mathcal{M}$.
We say $f$ is \emph{one-to-one} or \emph{injective} when its fiber,
$f^{-1}(y)$ is the singleton set $\{x\}$.

We let $A \oplus \eps := \{ x \in \mathbb{R}^N | \text{d}( x, A ) < \eps \} $ denote the $\eps$-neighborhood of the nonempty set $A$. 
In $\mathbb{R}^N$, 
we note the $\eps$-neighborhood of the nonempty set $A$ is the Minkowski sum of $A$ with $B_{\epsilon}^{N}(x)$, 
where the Minkowski sum between two sets $A$ and $B$ is:
$A \oplus B = \{ a + b | a \in A, b \in B \} $.  
For example,
an $n$ dimension open ball with radius $r$,
centered at a point $x$ can be expressed as:
$B_{r}^{n}(x)
= x \oplus B_{r}^{n}(0)
= x \oplus r $,
where the last expression is used to simplify notation.
If not specified, the dimension of the ball is $n$. 
We also use $B_r$ to denote the ball with radius $r$ when its center is irrelevant. 
Similarly, $S_r^n$ denotes $n$-dimensional sphere in $\real^{n+1}$ with radius $r$.
Let $\Vol_{n}$ denote $n$-dimensional volume.\footnote{ \label{fnt:n-volume}
	Let $A$ be a set.
	In Euclidean space,
	$\Vol_{n}(A) = \mathcal{L}^n(A) $ is the Lebesgue measure.
	For a general n-rectifiable set,
	$\Vol_{n}(A) = \mathcal{H}^n(A) $ is the Hausdorff measure.
	When $A$ is not rectifiable,
	$\Vol_{n}(A) = \mathcal{M}^n_{*}(A) $ is the lower Minkowski content. } 
When the intrinsic dimension of $A$ is greater than $n$, we set $\Vol_n(A) = \infty$.
Through the rest of the paper, 
we use $U$ to denote $B_{r_U}(x)$ a ball with radius $r_U$ centered at $x$ and $V = B_{r_V}(y)$ a ball with radius $r_V$ centered at $y$.
These are metric balls in a metric space. 
For example, 
they are geodesic balls in a Riemannian manifold, 
whenever they are well defined.
In Euclidean spaces, 
$U$ is a Euclidean ball with $L_2$ norm. 
By $T_{\#}(\mu) = \nu$, 
we mean a map $T$ pushes forward a measure $\mu$ to $\nu$, 
i.e.
$ \nu(B) = \mu (T^{-1}(B)) $
for any Borel set $B $.
We say a measure $\mu$ is dominated by another measure $\nu$, if for every measurable set $A$, $\mu(A) \le \nu(A)$. \section{Precision and recall}
\label{sec:mainresults}
We present the definitions of precision and recall in a continuous setting in this section. We then prove the equivalence between perfect recall and the continuity, 
followed by a theorem on the necessary tradeoff between the perfect recall and the perfect precision for a dimension reduction information retrieval system.
The main result of this section is a theoretical upper bound for the precision of a continuous DR map.

\subsection{Precision and recall}
\label{subsec:PreAndRec}
While precision and recall are commonly defined based on finite counts in practice, 
when analysing DR maps between spaces,
it is natural to extend their definitions in a continuous setting as follows. 
\begin{definition}[Precision and Recall]
	\label{def:set_precision_recall}
	Let $f$ be a continuous DR map. 
	Fix $(x, y = f(x))$, $r_U > 0$ and $r_V > 0$, let $U = \ball{r_U}{x} \subset \real^N$ and $V =B_{r_V}^m(y) \subset \real^m$ be the balls with radius $r_U$ and $r_V$ respectively. 
	The \textbf{precision} and \textbf{recall} of $f$ at $U$ and $V$ are defined as:
	\[
	\text{Precision}^{f} (U, V) = \frac{\Vol_{n}(f^{-1}(V) \cap U)} {\Vol_{n}(f^{-1}(V))}; \qquad 
	\text{Recall}^{f} (U, V) = \frac{\Vol_{n}(f^{-1}(V) \cap U)} {\Vol_{n}(U)}.
	\] 
	We say $f$ achieves \textbf{perfect precision} at $x$ if for every $r_U$, there exists $r_V$ such that $Precision^{f} (U, V) = 1$. Also, $f$ achieves \textbf{perfect recall} at $x$ if for every $r_V$, there exists $r_U$ such that $Recall^{f} (U, V) = 1$.
	Finally, we say $f$ achieves \textbf{perfect precision} (\textbf{perfect recall}, respectively) in an open set $W$, 
	if $f$ achieves perfect precision (perfect recall, respectively) at $w$ for any $w\in W$. 
	\end{definition}

\begin{figure}
\centering
\includegraphics[width=0.7\linewidth]{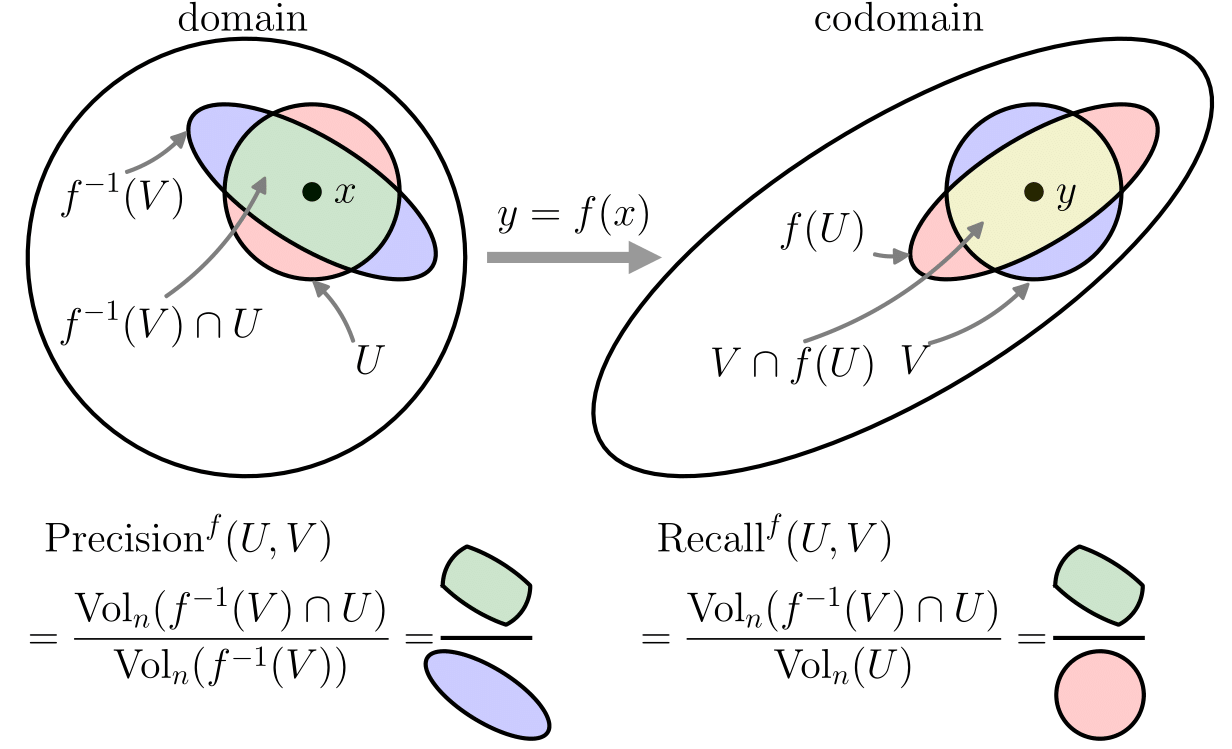}
\caption{Illustration of precision and recall.}
\label{fig:prec_rec}
\end{figure}

Note that perfect precision requires $f^{-1}(V) \subset U$ except a measure zero set. Similarly, perfect recall requires $ U \subset f^{-1}(V)$ except a measure zero set.
Figure \ref{fig:prec_rec} illustrates the precision and recall defined above.
To measure the performance of the information retrieval system, we would like to understand how different $f^{-1}(V)$ is from the ideal response $U = \ball{r_U}{x}$. Precision and recall provides two meaningful measures for this difference based on their volumes.
Note that $f$ achieves perfect precision at $x$ implies that no matter how small the relevant radius $r_U$ is for the image, the system would be able to achieve 0 false positive by picking proper $r_V$. Similarly perfect recall at $x$ implies no matter how small $r_V$ is, the system would not miss the most relevant images around $x$. 

In fact, 
the definitions of perfect precision and perfect recall are closely related to continuity and injectivity of a function $f$. 
Here we only present an informal statement. 
Rigorous statements are given in the Appendix \ref{appsec:prereconecon}. 
\begin{prop}
	Perfect recall is equivalent to continuity. If $f$ is continuous, then perfect precision is equivalent to injectivity.
\label{prp:equivalence_info_topology}
\end{prop}

The next result shows that no DR map $f$, 
continuous or not, 
can achieve perfect recall and perfect precision simultaneously - a widely observed but unproved phenomenon in practice.   
In other words, 
it rigorously justifies the intuition that perfectly maintaining the local neighbourhood structure is impossible for a DR map.
\begin{theorem}[Precision and Recall Tradeoff]
	\label{thm:strong_precision_recall_tradeoff}
	Let $n > m$, $\mathcal{M} \subset \mathbb{R}^N $ be a Riemannian $n$-dimensional submanifold. 
	Then for any (dimensionality reduction) map $f: \mathcal{M} \to \mathbb{R}^m$ and any open set $W \subset \cM$,
	$f$ cannot achieve both perfect precision and perfect recall on $W$.  
\end{theorem}

\subsection{Upper bound for the precision of a continuous DR map}
\label{subsec:upperboundprecision}
In this section, 
we provide a quantitative analysis for the imperfection of $f$.
In particular, we prove an upper bound for the precision of a continuous DR map $f$ (thus $f$ achieves perfect recall).
For simplicity, 
we assume the domain of $f$ is an $n$ -ball with radius $R$ embedded in $\real^N$, 
denoted by $B_{R}^{n}$. 
Our main tool is the Waist Inequality \citep{rayon2003isoperimetry, akopyan2017tight} in quantitative topology. See \cref{sec:waistinequality} for an exact statement. 

Intuitively, 
the Waist Inequality guarantees the existence of $y\in \real^m$ such that $f^{-1}(y)$ is a `large' fiber. 
If $f$ is also $L$-Lipschitz, 
then for $p$ in a small neighbourhood $V$ of $y$, $f^{-1}(p)$ is also a `large' fiber, 
thus $f^{-1}(V)$ has a positive volume in $\cM$. 
Exploiting the lower bound for $\Vol_{n}\left(f^{-1}(V)\right)$ leads to our upper bound in \cref{thm:precision_bound_ball} on the precision of $f$, $\fprec(U,V)$. 
A rigorous proof is given in the appendix \cref{sec:proofofupperbound}.
\begin{theorem}[Precision Upper Bound, Worst Case]
\label{thm:precision_bound_ball}
	Assume $n > m$, and that $f: B_{R}^{n} \to \mathbb{R}^m$ is a continuous map with Lipschitz constant $L$. 
	Let $r_U $ and $ r_V >0$ be fixed. 
	Denote 
	\begin{equation}
	\label{eq:dnm}
		D(n, m) 
		= 
		\frac{\Gamma(\frac{n-m}{2} + 1) \Gamma(\frac{m}{2} + 1) }{ \Gamma(\frac{n}{2} + 1) }\,.
	\end{equation}
Then there exists $y \in \real^m$ such that 
for any $x\in f^{-1}(y)$, we have: 
	\begin{equation}
	\label{eq:mainresult}
		Precision^{f}(U, V) \leq D(n, m)\,\left(\frac{r_U}{R}\right)^{n-m}\,\frac{ r_U^{^{m}} }{p^{m}(r_V/L)}
	\end{equation}
where $p^{m}(r)$ is $r^{m}\left(1 + o(1)\right)$,
i.e. 
$ \displaystyle{ \lim_{r \to 0 } \frac{p^{m}(r)}{r^{m}} = 1 } $. 
\end{theorem}
\begin{rmk}
	Key to the bound is the waist inequality. 
	As such, 
	upper bounds on precision for other spaces (i.e. cube, see \citet{klartag2017convex} ) can be established, 
	provided there is a waist inequality for the space. 
	The Euclidean norm setting can also be extended to arbitrary norms, 
	exploiting convex geometry (i.e. \citet{akopyan2018waist}).
	Rigorous proofs are given in the appendix \ref{sec:proofofupperbound}.
							\end{rmk}
\begin{rmk}
	With $m$ fixed as a constant, 
	note that $D(n, m)$ decays asymptotically at a rate of $(1/n)^{m/2}$. 
	Also note that $r_U<R$ implies $\left(\frac{r_U}{R}\right)^{n-m}$ decays exponentially. 
	Typically, $L$ can grow at a rate of $\sqrt{n}$. 
	Moreover, 
	while $p^m(r)$'s behaviour is given asymptotically, it is independent of $n$. 
	Thus the upper bound decay is dominated by the exponential rate of $n-m$. 
			For fixed $n, m$, 
	this upper bound can be trivial when $r_U \gg r_V$. 
	However, 
	this rarely happens in practice in the information retrieval setting.
	Note that the number of relevant items, which is indexed by $r_U$, is often smaller
	than the number of retrieved items, that depends on $r_V$, while they are both much smaller than number of total items, indexed by $R$.

	We note however that this bound depends on the intrinsic dimension $n$.
	When $n \ll N$ and the ambient dimension $N$ is used in place, 
	the upper bound could be misleading in practice as it is much smaller than it should be.
	To estimate this bound in practice, 
	a good estimate on intrinsic dimension \citep{granata2016accurate} is needed, 
	which is an active topic in the field and beyond the scope of this paper.
\end{rmk}

\cref{thm:precision_bound_ball} guarantees the existence of a particular point $y\in \real^m$ where the precision of $f$ on its neighbourhood is small. 
It is natural to ask if this is also true in an average sense for every $y$.
In other words, 
we know a information retrieval system based on DR maps always has a blindspot, 
but is this blindspot behaviour a typical case?
In general, when $m > 1$, this is false, due to a recent counter-example constructed by \citet{alpert2015family}. 
However, our next result shows that for a large number of continuous DR maps in the field, such upper bound still holds with high probability.

\begin{theorem}[Precision Upper Bound, Average Case]
	\label{thm:precision_bound_avg}
	Assume $n > m$ and $B_{R}^{n}$ is equiped with uniform probability distribution. 
	Consider  the following cases:
	\begin{itemize}
		\item \text{\bf case 1: } $m=1$ and $f: B_{R}^{n} \to \mathbb{R}^m$ is $L$ Lipschitz continuous, or
		\item \text{\bf case 2: } $f: B_{R}^{n} \to \mathbb{R}^m$ is a $k$-layer feedforward neural network map with Lipschitz constant $L$, with surjective linear maps in each layer.  
					\end{itemize} 
	Let $ 0 < \delta^2 < R^2 - r_U^2 $, $r_U, r_V > 0$ be fixed, then with probability at least $q_1$ for case 1 or $q_2$ for case 2, 
	it holds that 
	\begin{equation}
	\label{eq:mainresult_avg}
		Precision^{f}(U, V) \leq D(n, m)\,\left(\frac{r_U}{ \sqrt{r_U^2 + \delta^2} }\right)^{n-m}\,\frac{ r_U^{^{m}} }{p^{m}(r_V/L)},
	\end{equation}
					where 
	\[
	q_1 = \frac{ \frac{1}{2 \pi R} \int_{ B^m_{ \Re } } \Vol_{n-m+1} \Proj_1^{-1}(t) \text{d}t }{ \Vol_n(B^n_R) } \,, \quad 
	q_2 = \frac{ \int_{ B^m_{ \Re } } \Vol_{n-m} \Proj_2^{-1}(t) \text{d}t }{ \Vol_n(B^n_R) } \,,
	\]
	$\Re = \sqrt{ R^2 - r_U^2 - \delta^2 }$,  $\Proj_1: S_{R}^{n+1} \to \mathbb{R}^m$ and $\Proj_2: B_{R}^{n} \to \mathbb{R}^m$ are arbitrary surjective linear maps.
	Furthermore,
	\[
	\Lim{ \frac{ r_U^2 + \delta^2}{R^2} \to 0  } q_1 = 1 \quad \Lim{ \frac{ r_U^2 + \delta^2}{R^2} \to 0  } q_2 = 1.
	\]	
	\todor[]{Need another q value for the case $m=1$.}
\end{theorem}

See Appendix \cref{sec:precision_avg_case} for an explicit characterization of $\Proj_1^{-1}(t)$ and $\Proj_2^{-1}(t)$. 
\cref{thm:precision_bound_ball} and \cref{thm:precision_bound_avg} together suggest that practioners should be cautious in applying and interpreting DR maps. 
One important application of DR maps is in data visualization. 
Among the many algorithms, t-SNE's empirical success made it the de facto standard. 
While \citep{pmlr-v75-arora18a} shows t-SNE can recover inter-cluster structure in some provable settings, 
the resulted intra-cluster embedding will very likely be subject to the constraints given in our work \footnote{ Technically speaking, the DR maps induced by t-SNE may not be continuous, and hence our theorems do not apply directly. 
However, 
since we can measure how closely parametric t-SNE (which is continuous) behaves as t-SNE and there is empirical evidence to their similarity \citep{maaten2009learning}, 
our theorems may apply again. }. 
For example, recall within a cluster will be good, 
but the intra-cluster precision won’t be. 
In more general cases and/or when perplexity is too small, 
t-SNE can create artificial clusters, 
separating neighboring datapoints. 
The resulted visualization embedding may enjoy higher precision, 
but its recall suffers. 
The interested readers are referred to \cref{sec:additional_compare_maps} for more experimental illustrations. 
Our work thus sheds light on the inherent tradeoffs in any visualization embedding. 
It also suggests the companion of a reliability measure to any data visualization for exploratory data analysis, which measures how a low dimensional visualization represents the true underlying high dimensional neighborhood structure.\footnote{Such attempts existed in literature on visualization of dimensionality reduction (e.g. \citep{venna2010information}). However, since these works are based on heuristics, it is less clear what they measure, nor do they enjoy theoretical guarantee.}

\section{Wasserstein measure}
\label{sec:wassersteinmeasure}
\vspace{-1 mm}
Intuitively we would like to measure how different the original neighbourhood $U$ of $x$ is from the retrieved neighbourhood $f^{-1}(V)$ when using the neighbourhood of $f(x)$ in $\real^m$.
Precision and Recall in \cref{subsec:PreAndRec} provide a semantically meaningful way for this purpose and we gave a non-trivial upper bound for precision when the feature extraction is a continuous DR map.
However, precision and recall are purely volume-based measures.
It would be more desirable if the measure could also reflect the information about the distance
distortions between $U$ and $f^{-1}(V)$.
In this section, we propose an alternative measure to reflect such information based on the $L_2$-Wasserstein distance. 
Efficient algorithms for computing the empirical Wasserstein distance exists in the literature \citep{altschuler2017near}.
Unlike the measure proposed in \citet{venna2010information}, our measure also enjoys a theoretical guarantee similar to \cref{thm:precision_bound_ball},  
which provides a non-trivial characterization for the imperfection of dimension reduction information retrieval. 

Let $\Prob_U$ ($\Prob_{f^{-1}(V)}$, respectively) denote the uniform probability distribution over $U$ ($f^{-1}(V)$, respectively), and $\Xi(\mathbb{P}_{U}, \mathbb{P}_{f^{-1}(V)})$ be the set of all the joint distribution over $B_R^n \times B_R^n$, 
whose marginal distributions are $\Prob_U$ over the first $B_R^n$ and $\Prob_{f^{-1}(V)}$ over the second $B_R^n$.
We propose to measure the difference between $U$ and $f^{-1}(V)$ by the $L_2$-Wasserstein distance between  $\Prob_U$ and $\Prob_{f^{-1}(V)}$: 
\[
	W_{2}(\mathbb{P}_{U}, \mathbb{P}_{f^{-1}(V)}) = \inf_{\xi \in \Xi(\mathbb{P}_{U}, \mathbb{P}_{f^{-1}(V)})} \mathbb{E}_{(a, b) \sim \xi} [ \| a - b \|^{2}_{2} ]^{1/2}.
\]
In practice, it is reasonable to assume that $\text{Vol}_n(U)$ is small in most retrieval systems.
In such cases, low $W_2(P_U, P_{f^{-1}(V)})$ cost is closely related to high precision retrieval. 
To see that, when $\text{Vol}_n(U)$ is small, achieving high precision retrieval requires  small $\text{Vol}_n (f^{-1}(V))$, which is a precise quantitative way of saying $f$ being roughly injective.
Moreover, as seen in \Cref{subsec:PreAndRec}, $f$ being roughly injective $\approx$ $f$ giving high precision retrieval. 
As a result, we can expect high precision retrieval performance when optimizing $W_2(P_U, P_{f^{-1}(V)})$ measure. 
Such relation is also empirically confirmed in the simulation in \Cref{sec:simulation}.

Besides its computational benefits, for a continuous DR map $f$, the following theorem provides a lower bound on $W_{2}(\mathbb{P}_{U}, \mathbb{P}_{f^{-1}(V)})$ 
with a similar flavour to the precision upper bound in \cref{thm:strong_precision_recall_tradeoff}. 
\begin{theorem}[Wasserstein Measure Lower Bound]
 	\label{thm:wass_lowerbound}
 	Let $n > m$,
 	$f: B_{R}^{n} \to \mathbb{R}^m$ be a $L$-Lipschitz continuous map, 
 	where $R$ is the radius of the ball $B_{R}^{n}$.
 	There exists $y\in \real^m$ such that for any $x\in f^{-1}(y)$, $r_U$ and $r_V >0$ such that $r \ge r_U$, 
 	$$ W_{2}^2(\mathbb{P}_U, \mathbb{P}_{f^{-1}(V)}) 
 	\geq
 	\frac{n}{n+2} \left(r - r_U \right)^2
 	$$
 	where $r
 	= 
 	\left(\frac{\Gamma(\frac{n}{2} + 1)}{ \Gamma(\frac{n-m}{2} + 1) \Gamma(\frac{m}{2} + 1) }\right)^{ \frac{1}{n} } R^{\frac{n-m}{n}} \left(p^{m}(r_{V}/L)\right)^{ \frac{1}{n} }$. 
 	In particular, 
 	as $n \rightarrow \infty$,
 	\[
 	W_{2}^2(\mathbb{P}_U, \mathbb{P}_{f^{-1}(V)}) = \Omega\left((R-r_U)^2\right).
 	\]
 \end{theorem}
We sketch the proof here. A complete proof can be found in \cref{sec:W2proofs}. The proof starts with a lower bound of $\Vol_{n} \left(f^{-1}(V)\right)$ by the topologically flavoured waist inequality (\cref{eq:waistineq}). 
Heuristically $\Vol_{n} (f^{-1}(V))$ is much larger than $\Vol_{n} (U)$ when $n\gg m$ and $R\gg r_U$.
The main component of the proof is to establish an explicit lower bound for $W_{2}(\mathbb{P}_U, 
\mathbb{P}_{W})$ over all possible $W$ of a fixed volume $\cV$, 
\footnote{An antecedent of this problem was studied in Section 2.3 of \cite{mccann2004exact}, where the authors optimize over the more restricted class of ellipses with fixed area.
For our purpose, the minimization is over bounded measurable sets.}  
where $U$ is a ball with radius $r_U$, as shown in \cref{thm:optimalball}. 
In particular, 
we prove that the shape of optimal $W^*$ must be rotationally invariant, 
thus $W^*$ must be a union of spheres. 
This is achieved by levering the uniqueness of the solution to the optimal partial transport problem \citep{caffarelli2010free, figalli2010optimal}. 
We then prove that the optimal solution for $W$ is the ball that has a common center with $U$.  
\todog{Don't understand what are we trying to say here.}

\begin{theorem}
	\label{thm:optimalball}
	Let $U = B_{r_U}$ and $\cV \ge \Vol(U)$. 
	Then
	\[
	 \inf_{W:\,\Vol_n(W) \ge \cV} W_{2}(\mathbb{P}_U, \mathbb{P}_{W}) 
	 = 
	 \inf_{W:\,\Vol_n(W) = \cV} W_{2}(\mathbb{P}_U, \mathbb{P}_{W}) = W_{2}(\mathbb{P}_U, \mathbb{P}_{B_{r_{\cV}}}),
	\]
	where $B_{r_{\cV}}$ is an ${r_{\cV}}$ ball with the same center with $U$ such that $\Vol_n(B_{r_{\cV}}) = \cV$.
	Moreover, $T(x) =\frac{r_U}{r_{\cV}}x$, for $x\in B_{r_{\cV}}$ is the optimal transport map (up to a measure zero set), 
	so that
	\[
	W_{2}(\mathbb{P}_U, \mathbb{P}_{B_{r_{\cV}}}) = \int_{B_{r_{\cV}}} | x - T(x) |^2 \,\td \Prob_{B_{r_{\cV}}} (x).
	\]
	Complementarily, 
	when $0 < \cV < \Vol_n(U)$, the infimum $\inf_{W:\,\Vol_n(W) = \cV} W_{2}(\mathbb{P}_U, \mathbb{P}_{W}) = 0$, is not attained by any set. 
	On the other hand, $\inf_{W:\,\Vol_n(W) \ge \cV} W_{2}(\mathbb{P}_U, \mathbb{P}_{W}) = 0$ by taking $W = U$.
\end{theorem}
\begin{rmk}
	\label{rmk:tightbounds}
	Our lower bound in \cref{thm:wass_lowerbound} is (asymptotically) tight. Note that by \cref{thm:wass_lowerbound}, 
	$W_{2}^2(\mathbb{P}_U, \mathbb{P}_{f^{-1}(V)}) $ has a (maximum) lower bound of scale $(R-r_U)^2$. 
	On the other hand, 
	by \cref{thm:optimalball}, $ W_{2}^2(\mathbb{P}_U, \mathbb{P}_{f^{-1}(V)}) \leq W_{2}^2(\mathbb{P}_U, \mathbb{P}_{B^n_R}) = \Omega ((R - r_U)^2)$, where the equality is by standard algebraic calculations. 
\end{rmk}

\subsection{Iso-Wasserstein inequality }
We believe \cref{thm:optimalball} is of independent interest itself, as it has the same flavor as the isoperimetric inequality (See \cref{sec:waistinequality} for an exact statement.) which arguably is the most important inequality in metric geometry.
In fact, the first statement of \cref{thm:optimalball} can be restated as the following inequality: 
\begin{theorem}[Iso-Wasserstein Inequality]
	\label{thm:isowasserstein}
	Let $B_{r_1}, B_{r_2} \subset B^n_R$ be two concentric $n$ balls with radii $r_1 \leq r_2$ centered at the origin. 
	For all measurable $A \subset B_R^n $ with
	$ \Vol_{n} (A) = \Vol_{n} (B_{r_2}) $, 
	we have 
	$$ W_2(\mathbb{P}(A), \mathbb{P}(B_{r_1})) \geq W_2(\mathbb{P}(B_{r_2}), \mathbb{P}(B_{r_1})) $$
	where $\mathbb{P} (S)$ denotes a uniform probability distribution on $S$, i.e. $\mathbb{P}(S)$ has density $\frac{1}{\Vol_{n}(S)}$. 
\end{theorem}
Recall that an isoperimetric inequality in Euclidean space roughly says balls have the least perimeter among all equal volume sets.
Theorem \ref{thm:isowasserstein} acts as a transportation cousin of the isoperimetric inequality. 
While the isoperimetric inequality compares $n-1$ volume between two sets,  
the iso-Wasserstein inequality compares their Wasserstein distances to a small ball. 
The extrema in both inequalities are attained by Euclidean balls.

\subsection{Simulations}
\label{sec:simulation}
In this section, we demonstrate on a synthetic dataset that our lower bound in \cref{thm:wass_lowerbound} can be a reasonable guidance for selecting the retrieval neighborhood radius $r_V$, which emphasizes on high precision.
The simulation environment is to compute the optimal $r_V$ by minimizing the lower bound in Theorem \ref{thm:wass_lowerbound}, 
with a given relevant neighborhood radius $r_U$ and embedding dimension $m$. 
Note that minimizing its lower bound instead of the exact cost itself is beneficial as it avoids the direct computation of the cost. 
Recall the lower bound of $W_2(P_U, P_{f^{-1}(V)})$ is (asymptotically) tight (\cref{rmk:tightbounds}) and matches the its upper bound when $n-m \gg 0$.  
If the lower bound behaves roughly like $W_2(P_U, P_{f^{-1}(V)})$,
our simulation result also serves as an empirical evidence that $W_2(P_U, P_{f^{-1}(V)})$ weighs more on high precision.

Specifically, 
we generate 10000 uniformly distributed samples in a 10-dimensional unit $\ell_2$-ball.
We choose $r_U$ such that on average each data point has 500 neighbors inside $B_{r_U}$. 
We then linearly project these 10 dimensional points into lower dimensional spaces with embedding dimension $m$ from 1 to 9.
For each $m$, a different $r_V$ is used to calculate discrete precision and recall.
This simulates how optimal $r_V$ according to Wasserstein measure changes with respect to $m$. The result is shown in on the left in Figure \ref{fig:simulation}. 
Similarly, we can fix $m = 5$ and track optimal $r_V$'s behavior when $r_U$ changes. This is shown on the right in Figure \ref{fig:simulation}.

We evalute our measures based on traditional information retrieval metrics such as f-score. 
To compute it,
we need the discrete/sample-based precision and recall. 
As discussed in the introduction,
a naive sample based calculations of precision and recall makes $ Precision = Recall $ at all times. 
We compute them alternatively by discretizing \cref{def:set_precision_recall}, by fixing radii $r_U$ and $r_V$.
So each $U$ and $f^{-1}(V)$ contain different numbers of neighbors.  
\begin{equation}
\label{eqn:dist_precision}
	Precision = \frac{\#(\mathrm{points~within~r_U~from}~x~\mathrm{and~within~r_V~from}~y)}{\#(\mathrm{points~within~r_V~from}~y)}
\end{equation}
\begin{equation}
\label{eqn:dist_recall}
	Recall = \frac{\#(\mathrm{points~within~r_U~from}~x~\mathrm{and~within~r_V~from}~y)}{\#(\mathrm{points~within~r_U~from}~x)}
\end{equation}

The optimal $r_V$ according to the lower bound in \cref{thm:wass_lowerbound} 
(the blue circle-dash-dotted line) aligns closely with the optimal f-score with $\beta=0.3$
where $\beta$ weighted f-score, also known as f-$\beta$score, is:
\[
	(1+\beta^2)\frac{Precision*Recall}{\beta^2 * Precision + recall}.
\]
Note that f-score with $\beta < 1$ indeed emphasizes on high precision.

In this provable setting, we have demonstrated our bound's utility. 
This shows $W_2$ measures' potential for evaluating dimension reduction. 
In general cases, we won't have such tight lower bounds and it is natural to optimize according to the sample based $W_2$ measures instead. 
We performed some preliminary experiments on this heuristic, shown in Appendix \ref{sec:wasserstein_experiments}.

\begin{figure}[t]
\centering
\includegraphics[width=14cm,height=7cm]{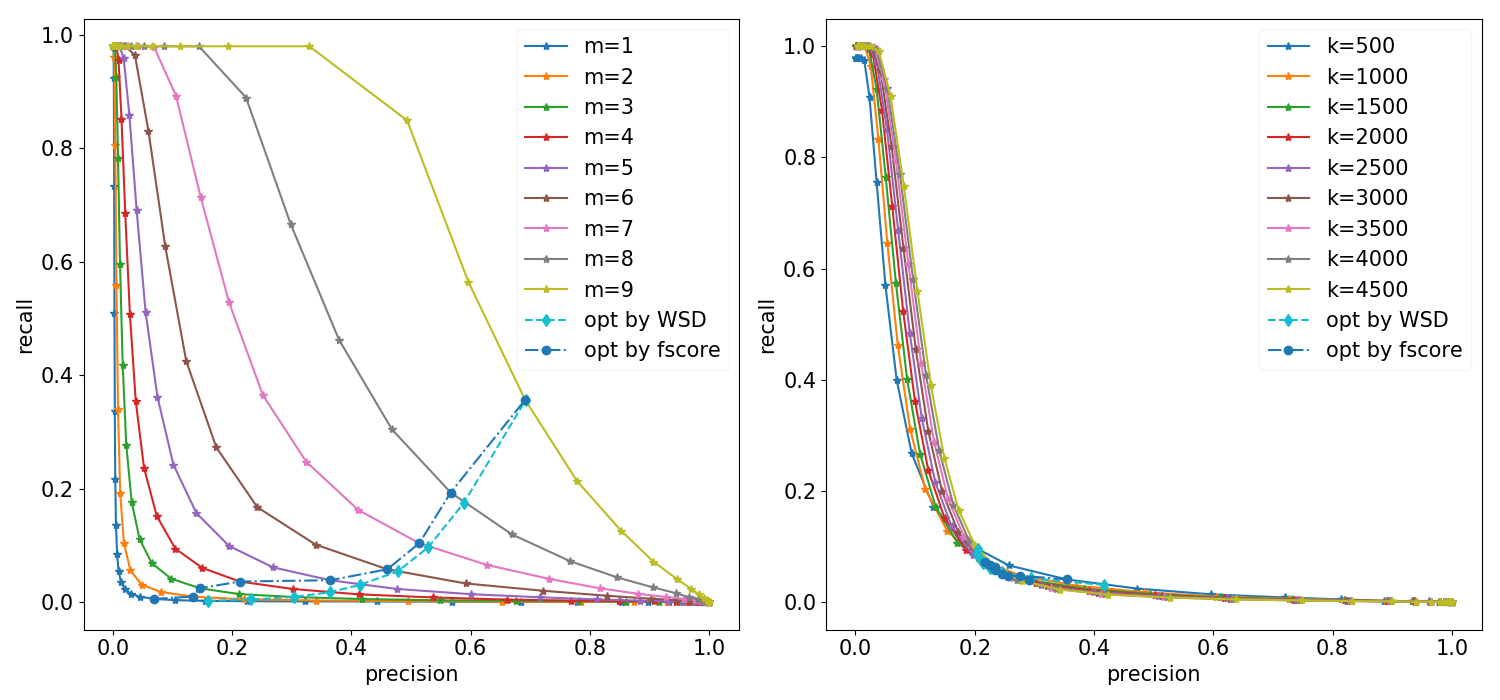}
\caption{Precision and recall results on uniform samples in a 10 dimensional unit ball. 
The left figure contains precision-recall curves for a fixed $r_U$ and the optimal $r_V$ is chosen according to $m = 1, \cdots, 9$. 
The right figure plots the curves for $m = 5$ and the optimal $r_V$'s is chosen for different $r_U$, where $r_U$ is indexed by $k$, the average number of neighbors across all points.}
\label{fig:simulation}
\end{figure}

\vspace{-1 mm}

\section{Relation to metric space embedding and manifold learning}
\label{sec:metric_manifold}
\vspace{-2 mm}
We lastly situate our work in the lines of research on metric space embedding and manifold learning. 
One obvious difference between our work and the literature of metric space embedding and manifold learning is that
our work mainly focuses on intrinsic dimensionality reduction maps, i.e. $n\gg m$, while in metric space embedding and manifold learning, having $n\le m < N $ is common.

Our work also differs from the literature of metric space embedding and manifold learning in its learning objective. 
Learning in these fields aims to preserve the metric structure of the data.
Our work attempts to preserve precision and recall, 
a weaker structure in the sense of embedding dimension (\cref{prop:manifoldlearning}). 
While they typically look for lowest embedding dimension subject to certain loss (e.g. smoothness, local or global isometry), 
in contrast, our learning goal is to minimize the loss (precision and recall etc.) subject to a fixed embedding dimension constraint.
In these cases, desired structures will break (\cref{thm:precision_bound_avg}) because we cannot choose the embedding dimension $m$ (e.g. for visualizations $m = 2$; for classifications $m = \text{number of classes}$). 
\todor[]{Kry double check this paragraph} 

We now discuss the technical relations with metric space embedding and manifold learning. 
Many datasets can be modelled as a finite metric space $\mathcal{M}_k$ with $k$ points. 
A natural unsupervised learning task is to learn an embedding that approximately preserves pairwise distances. 
The Bourgain embedding \citep{bourgain1985lipschitz} guarantees the metric structure can be preserved with distortion $O(\log k)$ in $l_{p}^{O(\log^2k)}$.
When the samples are collected in Euclidean spaces, i.e. $\mathcal{M}_k \subset l_2$,
the Johnson-Lindenstrauss lemma
\citep{dasgupta2003elementary} improves the distortion to (1 + $\epsilon$) in $l_{2}^{O(\log(k/\epsilon^2))}$. 
These embeddings approximately preserve all pairwise distances - global metric structure of $\mathcal{M}_k$ is compatible to the ambient vector space norms.
Coming back to our work, 
it is natural to mimic this approach for precision and recall in $\mathcal{M}_k$. 
The first problem is that the naive sample based precision and recall are always equal (\cref{sec:simulation}).
A second problem is discrete precision and recall is a non-differentiable objective. 
In fact, 
the difficulty of analyzing discrete precision and recall motivates us to look for continuous analogues.

Roughly, our approach is somewhat similar to manifold learning where researchers postulate that the data $\mathcal{M}_k$ are sampled from a continuous manifold $\mathcal{M}$, 
typically a smooth or Riemannian manifold $\mathcal{M}$ with intrinsic dimension $n$. 
In this setting, 
one is interested in embedding $\mathcal{M}$ into $l_2$ locally isometrically.
Then one designs learning algorithms that can combine the local information to learn some global structure of $\mathcal{M}$. 
By relaxing to the continuous cases just like our setting, 
manifold learning researchers gain access to vast literature in geometry. 
By the Whitney embedding \citep{mcqueen2016nearly}, $\mathcal{M}$ can be smoothly embedded into $\mathbb{R}^{2n}$.
By the Nash embedding \citep{verma2013distance}, a compact Riemannian manifold $\mathcal{M}$ can be isometrically embedded into $\mathbb{R}^{p(n)}$, where $p(n)$ is a quadratic polynomial.
Hence the task in manifold learning is wellposed:
one seeks an embedding $f: \mathcal{M} \subset \mathbb{R}^{N} \to \mathbb{R}^{m}$ with $m \leq 2n \ll N$ in the smooth category or $m \leq p(n) \ll N$ in the Riemannian category.
Note that the embedded manifold metrics (e.g. the Riemannian geodesic distances) are not guaranteed to be compatible to the ambient vector space's norm structure with a fixed distortion factor,
unlike the Bourgain embedding or the Johnson-Lindenstrauss lemma in the discrete setting. 
A continuous analogue of the norm compatible discrete metric space embeddings is the Kuratowski embedding, 
which embeds global-isometrically (preserving pairwise distance) any metric space to an infinite dimensional Banach space $L^{\infty}$. 
With $\epsilon$ distortion relaxation, 
it is possible to embed a compact Riemannian manifold to a finite dimensional normed space. 
But this appears to be very hard, 
in that the embedding dimension may grow faster than exponentially in $n$ \citep{roeer2013finite}.

Like DR in manifold learning and unlike DR in discrete metric space embedding, rather than global structure we want to preserve local notions such as precision and recall.
Unlike DR in manifold learning, 
since precision and recall are almost equivalent to continuity and injectivity (\cref{thm:strong_precision_recall_tradeoff}), 
we are interested in embeddings in the topological category, instead of the smooth or the Riemannian category. 
Thus, 
our work can be considered as manifold learning from the perspective of information retrieval, 
which leads to the following result.

\begin{prop}
	\label{prop:manifoldlearning}
	If $m\ge 2n$, where $n$ is the dimension of the data manifold $\cM$ in domain and $m$ is the dimension of codomain $\real^m$, then there exists a continuous map $f:\,\cM\rightarrow \real^m$ such that $f$ achieves perfect precision and recall for every point $x\in \cM$.
\end{prop} 
Note that the dimension reduction rate is actually much stronger than the case of Riemannian isometric embedding where the lowest embedding dimension grows polynomially \citep{verma2013distance}.
This is because preserving precision and recall is weaker than isometric embedding. 
A practical implication is that, 
we can reduce many more dimensions if we only care about precision and recall.

\section{Conclusions}
\label{sec:DisCon}

We characterized the imperfection of dimensionality reduction mappings from a quantitative topology perspective.
We showed that perfect precision and perfect recall cannot be both achieved
by any DR map.
We then proved a non-trivial upper bound for precision for Lipschitz continuous DR maps.
To further quantify the distortion, 
we proposed a new measure based on $L_2$-Wasserstein distances,
and also proved its lower bound for Lipschitz continuous DR maps.
It is also interesting to analyse the relation between the recall of a continuous DR map and its modulus of continuity. 
However, 
the generality and complexity of the fibers (inverse images) of these maps so far defy our effort and this problem remains open. 
Furthermore, it is interesting to develop a corresponding theory in the discrete setting. 

\section{Acknowledgement}

We would like to thank Yanshuai Cao, Christopher Srinivasa, and the broader Borealis AI team for their discussion and support. We also thank Marcus Brubaker, Cathal Smyth, and Matthew E. Taylor for proofreading the manuscript and their suggestions, as well as April Cooper for creating graphics for this work.

\bibliographystyle{plainnat}
\bibliography{nldr}

\newpage

\appendix

\section{Waist Inequality and Isoperimetric Inequality}
\label{sec:waistinequality}

\begin{theorem}[Waist Inequality, \citet{akopyan2017tight}]
	\label{thm:waistinequality}
	Let $m \leq n$ and $f$ be a continuous map from the ball $B_{R}^n$ of radius $R$ to $\mathbb{R}^m$.
	Then there exists some $y \in \mathbb{R}^m$ such that 
	\[
	\Vol_{n-m}\left(f^{-1}(y)\right) \geq \Vol_{n-m} \left(B_R^{n-m}\right).
	\footnote{It is natural to consider $n-m$ dimensional volume for $f^{-1}(y) $,
		due to Sard's theorem \citep{guillemin2010differential} and implicit function theorem: since almost every $y \in f(B^n)$ is a regular value,
		$f^{-1}(y)$ is an $n-m$ dimensional submanifold,
		for such regular $y$.
																		For an arbitrary continuous function, $\Vol_{n-m} = \mathcal{M}^{n-m}_{*}$ is the lower Minkowski content, where the Waist Inequality is established \citep{akopyan2018waist}. For $n-m$ rectifiable sets, $\Vol_{n-m} = \mathcal{M}^{n-m}_{*} = \mathcal{H}^{n-m} $. }
	\]
	Moreover, for all $\eps > 0$: 
	\begin{equation}
	\label{eq:waistineq}
	\Vol_{n}\left(f^{-1}(y)\oplus \eps \right) 
	\geq 
	\frac{1}{2\pi R} \Vol_{n-m+1} \left(S_R^{n-m+1}\right)  \Vol_{m} \left(B^{m}_1\right) p^{m}(\eps),
	\end{equation}
	where $p^{m}(\eps)$ is $\eps^{m}\left(1 + o(1)\right)$, 
	i.e. $ \displaystyle{ \lim_{\eps \to 0 } \frac{p^{m}(\eps)}{\eps^{m}} = 1 } $,
	and $ f^{-1}(y) \oplus \eps$ denotes the set of points $ x \in B^{n}_R$ such that $ d(x, f^{-1}(y)) < \eps$,
	$S_R^{n-m+1}$ is the (n+m-1)-dimensional sphere of radius $R$, and $B^m_1$ is the unit $m$ ball.
\end{theorem}

\begin{rmk}
					When $m = 1$, 
	Waist Inequality generalizes classic concentration of measure on $B^n_R$,
	which says most volume of a high dimensional ball concentrates around its equator slab,
	as $n \rightarrow \infty$. 
	When $m > 1$, 
	we can roughly interpret the theorem as $f^{-1}(y)\oplus \eps$ is big in $n-m$ dimensions in the sense of volume, 
	thus it generalizes concentration of measure when $m > 1$. 
\end{rmk}

Intuitively the Waist inequality states that a higher dimensional space is too big in the sense of \textbf{volume}
that we cannot hope to squeeze it \textbf{continuously} into lower dimensional spaces,
without collapsing in some direction(s).
In other words, if an input domain is higher dimensional and thus in some sense large, 
then it must be large in at least one direction.
Waist inequality is a precise quantitative version of the topological invariance of dimension, 
which states balls of different dimensions cannot be homeomorphically mapped to each other. 
It is this mis-match between high and low dimensional nature of volumes that motivates us to formulate and prove the imperfection between precision and recall. 
A recent survey of the inequality can be found in \citep{guth2012waist}. 

\begin{theorem}[Isoperimetric Inequality]
	\label{thm:isoperimetric_inequality}
	
	Suppose $U \subset \mathbb{R}^n$ is a bounded (Hausdorff) measurable set, 
	with (Hausdorff) $n-1$ measurable boundary, 
	denoted as $ \Vol_{n-1} \partial U $.  
	Then:
	\[
	\Vol_n (U) = \Vol_n (B^n_1) \implies \Vol_{n-1} ( \partial U) \geq \Vol_{n-1} ( \partial B^n_1)
	\]
	Stated differently,
	\[ 
	\Vol_n ( U ) 
	\leq 
	\frac{1}{ n^{ \frac{n}{n-1} } \Vol_n (B_1)^{\frac{1}{n-1}} } \Vol_{n-1} ( \partial U )^{\frac{n}{n-1}}  
	\]
\end{theorem}

The first way of looking at the isoperimetric inequality is from an optimization viewpoint. 
It states that Euclidean balls are optimal sets in terms of minimizing the $n-1$ hypersurface volume, 
with a constraint on their $n$ volume. 
The second (equivalent) inequality is from an inequality angle.
It allows us to control the $n$ volume of a set in terms of its boundary's $n-1$ volume. 
For more information about this fundamental inequality,
we refer the reader to \citep{payne1967isoperimetric}.

\begin{figure}[t]
	\centering
	\includegraphics[width=\linewidth]{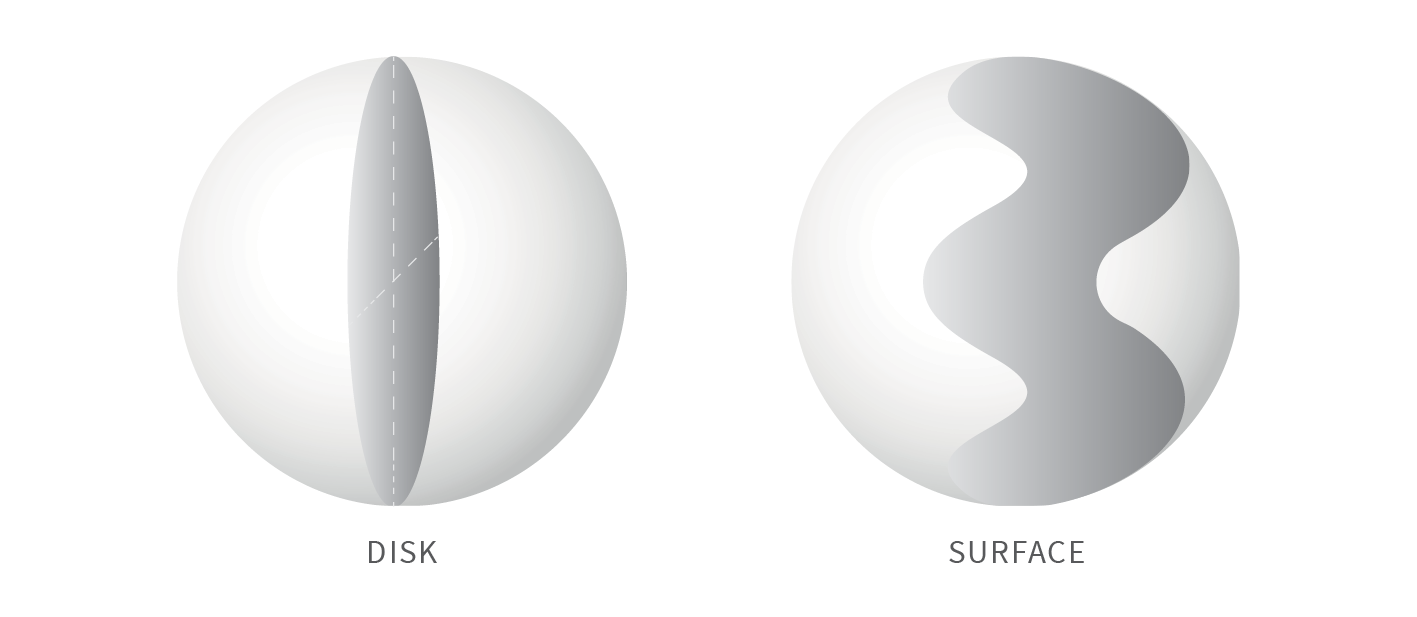}
	\caption{The above pictorial illustration compares $f^{-1}(y)$ - the pancake surface living in a 3-ball to a disc living in the 3-ball. We see that $f^{-1}(y)$ has bigger or equal area than the corresponding disc.}
	\label{fig:waist}
\end{figure}

Among all equal volume sets on the plane, the isoperimetric inequality says that the disc has the least perimeter. This statement compares all domains to balls. 
The waist inequality is its close cousin with perhaps stronger topological flavor. This is a statement about all continuous maps $f: B^{n}_R \rightarrow \mathbb{R}^m $: 
we can find $f^{-1}(y)$ such that $\Vol_{n-m} (f^{-1}(y)) \ge \Vol_{n-m} B^{n-m}_R $. This compares all continuous maps's volume-maximal fiber to balls. 
See \cref{fig:waist} for an illustration in 3D.

\section{Precision, Recall, One-To-One, and Continuity}
\label{appsec:prereconecon}
We extend the definitions of continuity and injectivity to allow exceptions on a measure zero set. 
For a dimensionality reduction map $f: \real^n \rightarrow \real^m$, 
we say it is essentially one-to-one if its `injectivity' is essentially no more than the reduction part. 
The manifold setting $f: \mathcal{M}^n \rightarrow \real^m$ is handled naturally by using coordinates and parametrization by open sets in $\real^n$, as in classical differential topology and differential geometry. 
\begin{definition}[Essential Continuity]
		\label{def:obs_continuity}
	$f$ is essentially continuous at $x$, if for any $\eps>0$, there exists $r>0$, such that for all the neighbourhood $U\ni x$ satisfying $\text{diam}(U) \le r$,
	\[\Vol_n\left(\{ u \in U\,: \, |f(u) - f(x)| >\eps \} \right) = 0.\]
	 	 We say $f$ is essentially continuous on a set $W$ if $f$ is essentially continuous at every $w\in W$.
\end{definition}
\begin{definition}[Essential Injectivity]
	\label{def:obs_one-to-one}
	 $f$ is essentially one-to-one or essentially injective at $x$, if for $y = f(x) \in \mathbb{R}^m$,
		$\Vol_{n - m}\left(f^{-1}(y)\right) = 0$\footnote{If the dimension of $f^{-1}(y)$ is greater than $n-m$, we define its volume to be $\infty$}. $f$ is essentially one-to-one on a set $W$ if $f$ is essentially one-to-one at every $w\in W$.
\end{definition}

Note that the definition of essential continuity (one-to-one, respectively) strictly generalizes the definition of continuity (one-to-one, respectively). 
In other words, 
every continuous function is essentially continuous, 
and there exists discontinuous functions that are essentially continuous.
The following lemma shows that if $f$ is essentially continuous on an open set $W$, 
then $f$ is continuous on $W$.
\begin{lem}[Essential continuity in a neighborhood]
	\label{lem:obs_continuity_neighbor}
	Essential continuity in a neighborhood and continuity in a neighborhood are equivalent. \todor[]{Why essential continuity is not defined as essential continuity at all the points in $W$ except a measure zero set? If this is the definition, then this theorem is not true.}
\end{lem}
\begin{proof}
It is sufficient to prove that if $f$ is essentially continuous on an open set $W$, 
then $f$ is continuous on $W$.
Assume that $f$ is not continuous on $W$, 
i.e., 
there exists $\eta > 0$, $w\in W$ and a sequence $\{w_1, \ldots, w_n,\ldots\}$ such that $\lim_{n\rightarrow \infty} w_n = w$, but $ |f(w_n) - f(w)| \geq \eta$.
Since $f$ is essentially continuous on $W$, there exists a neighbourhood of $w$,
$U\subset W$, 
such that $\Vol_n(E_U ) = 0$,
where $E_U = \{ u \in U\,: \, |f(u) - f(w)| >\eta/3 \}$. 
Note that for large enough $M$,
$w_M \in E_U$. 
Moreover, 
since $f$ is also essentially continuous at $w_M$, for a small neighbourhood $V$ of $w_M$, $\Vol_n(\{v\in V\, :\, |f(v) - f(w_M)| \le  \eta/3\}) = \Vol_n(V) >0$. However, note that this positive measure set $\{v\in V\, :\, |f(v) - f(w_M)| \le  \eta/3\}$ is a subset of $E_U$ by the definition of $E_U$, 
contradicting $\Vol_n(E_U) = 0$. 
\end{proof}
We next prove the equivalence between perfect recall and essential continuity.
\begin{prop}
	\label{prop:recallcontinuity}
For any map $f: \mathcal{M} \subset \mathbb{R}^N \to \mathbb{R}^m$, $f$ achieves perfect recall in an open set $W$, 
if and only if $f$ is essentially continuous on $W$.
\end{prop}
\begin{proof}
	\textbf{(Perfect Recall $\Rightarrow$ Essential Continuity)} 
	For any $x\in W$, any $\eps >0$, let $V = \{f(v) \in\real^m\, :\, |f(v) - f(x)| \le \eps  \}$. Since $f$ achieves perfect recall at $x$, there exists $r>0$, such that $\Vol_n( f^{-1}(V) \cap \ball{r}{x} ) = \Vol_n( \ball{r}{x})$. Therefore, for any $U$ such that $U\subset \ball{r}{x}$,
	\[
	\Vol_n\left(\{ u \in U\,: \, |f(u) - f(x)| >\eps \} \right)  \le \Vol_n\left( \{ u \in U\,: \, u\notin f^{-1}(V) \cap \ball{r}{x} \} \right) = 0.
	\] 
	Thus $f$ is essentially continuous at $x$.
	
	\textbf{(Essential Continuity $\Rightarrow$ Perfect Recall)}
	By \cref{lem:obs_continuity_neighbor}, $f$ is continuous on $W$. For any $x\in W$, assume $f(x)=y$. For any $r_V>0$, $f^{-1}(\ball{r_V}{y})$ is an open set in $\cM$. Therefore, 
	there exists small enough $r_U$ such that $ \ball{r_U}{x}\subset f^{-1}(V)$, 
	thus $\text{Recall}^{f} (\ball{r_U}{x}, \ball{r_V}{y}) = 1$.
\end{proof}

Based on this proposition, 
we can further prove that if $f$ is (essentially) continuous on $W$, 
then $f$ has neither perfect precision nor essential injectivity property on $W$.

\begin{prop}
	\label{prop:precisiononetoone}
	Let $f: \mathcal{M}^n \subset \mathbb{R}^N \to \mathbb{R}^m$, with $ m < n $.
	If $f$ is (essentially) continuous with approximate differential well defined on an open set $W$ almost everywhere, 
	\footnote{This is a weaker condition than Lipschitz, 
	including functions of bounded variation.
	A Lipschitz function is differentiable almost everywhere. 
			}, 
	then $f$ possesses neither perfect precision nor essential injectivity on $W$.
\end{prop}
\begin{proof}
	\noindent\textbf{(Continuous in neighborhood $\Rightarrow$ Not Essentially Injective)}
	We first prove that if $f$ is continuous on $W \subset \real^n$, 
	then $f$ is not essentially one-to-one on $W$. 
	To prove that $f$ does not have perfect precision, 
	it is sufficient to prove that the perfect precision of $f$ implies $f$ being essentially one-to-one.
	We handle the manifold case at the end of the proof, 
	by coordination: $\phi: U \subset \mathcal{M}^n \rightarrow V \subset \real^n$, 
	and parametrization $\phi^{-1}: V \subset \real^n \rightarrow U \subset \mathcal{M}^n $. 

	Assume $f$ is essentially one-to-one on $W$, 
	thus for any $y \in f(W) \subset \real^m$, 
	\[\Vol_{n-m}(f^{-1}(y)) = \int_{f^{-1}(y)} \text{d}\Vol_{n-m}(p) =0.\] 
	Since $W \subset \real^n $ is open, 
	there is an open ball $B^n_{\tau} \subset W $ such that we can consider the restriction of $f$ onto $B^n_{\tau}$. 
	Now \cref{thm:waistinequality} guarantees the existence of $y_\tau \in f(B^n_{\tau})$ such that 
	\[\Vol_{n-m}(f^{-1}(y_\tau)) \ge \Vol_{n-m} (B^n_{\tau}) > 0.\] 
	This contradiction completes the proof in the Euclidean case. 

	Now, for a map $f: W \subset \mathcal{M}^n \rightarrow \real^m $. 
	We consider the restriction of $f$ on $U \subset W$ where $U$ is homeomorphic to $\real^n$. 
	Then the composite map: 
	$f \circ \phi^{-1} \rightarrow \real^m $ is again a map between Euclidean spaces. 
	The argument above applies and we complete this part of the proof.

	\noindent\textbf{(Perfect Precision $\Rightarrow$ Essential One-to-one)}
	Assume that $f$ is not essentially one-to-one on $W$, 
	thus $f$ is not one-to-one on $W$. 
	Therefore, 
	there exist $y$, $z_1$, and $z_2$ such that $f(z_1) = f(z_2) = y$. 
	Without loss of generality, 
	assume $d(z_1, z_2) = 1$. Since $f$ has perfect precision, 
	picking $U = B^{m}_{0.4}(z_1)$, 
	there exists $r_{V,1}$, 
	such that $\Vol_n\left(f^{-1}(B^{m}_r(y)) \cap B^{m}_{0.4}(z_1)\right) = \Vol_n\left(f^{-1}(B^{m}_r(y))\right)$ for $r\le r_{V,1}$. 
	Similarly, 
	there exists $r_{V,2}$, 
	such that $\Vol_n\left(f^{-1}(B^{m}_r(y)) \cap B^{m}_{0.4}(z_2)\right) = \Vol_n\left(f^{-1}(B^{m}_r(y))\right)$ for $r\le r_{V,2}$. Further note that $B^{m}_{0.4}(z_1) \cap B^{m}_{0.4}(z_2) = \emptyset$. For $r \le \min\{r_{V,1}, r_{V,2}\}$, then
	\begin{align*}
	\Vol_n(f^{-1}(B^{m}_r(y))) & \ge   \Vol_n(f^{-1}(B^{m}_r(y)) \cap B^{m}_{0.4}(z_1)) + \Vol_n(f^{-1}(B^{m}_r(y)) \cap B^{m}_{0.4}(z_2)) \\
	&= 2*\Vol_n(f^{-1}(B^{m}_r(y))).
	\end{align*}
	Therefore, 
	$\Vol_n(f^{-1}(B^{m}_r(y))) = 0$. 
	Now since $f$ is continuous,
	$f^{-1}(B^{m}_r(y))$ is an open set in $\cM$, 
	thus $\Vol_n(f^{-1}(B^{m}_r(y)))$ cannot be 0, 
	a contradiction.
\end{proof}
Based on \cref{prop:recallcontinuity,prop:precisiononetoone}, the proof of \cref{thm:strong_precision_recall_tradeoff} is straightforward.
\begin{proof} [Proof of \cref{thm:strong_precision_recall_tradeoff}]
	It is sufficient to prove that if $f$ achieves perfection recall at $W$, then $f$ cannot achieve perfect precision at $W$. Since $f$ achieves perfect recall at $W$, by \cref{prop:recallcontinuity} $f$ is continuous, thus by \cref{prop:precisiononetoone} $f$ cannot achieve perfect precision at $W$.
\end{proof}

\section{Proof of \cref{thm:precision_bound_ball}}
\label{sec:proofofupperbound}
We present the proof of \cref{thm:precision_bound_ball} in this section. 
The following proposition develops a lower bound for the volume of the inverse image of $f$ on a particular small open set.
\begin{prop}
	\label{prop:lowerboundballvol}
	If $f$ is a continuous function with Lipschitz constant $L$, 
	then for any $y\in\real^m$ and $\eps>0$,
	\[
	\Vol_{n}\left(f^{-1}( B^m_{\eps}(y) )\right) 
	\ge 
	\Vol_{n}\left({f^{-1}(y) \oplus \frac{\eps}{L} }\right).
	\]
\end{prop}
\begin{proof}
	Since $f$ is Lipschitz, 
	for any $x$ such that $\text{d}(x, f^{-1}(y)) \le \frac{\eps}{L}$, $|f(x) - f(y)| \le \eps$. 
	Thus 
	\[
	f^{-1}(y) \oplus \frac{\eps}{L} = \{x\in \cM\,:\, \text{d}(x, f^{-1}(y)) \le \frac{\eps}{L}  \} \subset \{x\in \cM\,:\, |f(x) - f(y)| \le \frac{\eps}{L}  \} = f^{-1}\left( B^m_{\eps}(y) \right). 
	\]
	Therefore, 
	\[
	\Vol_{n}\left(f^{-1}( y \oplus \eps )\right) 
	\ge 
	\Vol_{n}\left( f^{-1}(y) \oplus \frac{\eps}{L} \right).
	\]
\end{proof}
\begin{proof}[Proof of \cref{thm:precision_bound_ball}]
	By \cref{thm:waistinequality}, 
	there exists $y\in \real^m$ such that 
	\[
	\Vol_{n}\left(f^{-1}(y)\oplus \eps \right) 
	\geq 
	\frac{1}{2\pi R} \Vol_{n-m+1} \left(S_R^{n-m+1}\right)  \Vol_{m} \left(B^{m}_1\right) \eps^m\left(1+o(1)\right).
	\]
	For any $x\in f^{-1}(y)$, $r_U, r_V >0$, 
	recall that 
	$\fprec(U, V) = \frac{\Vol_{n}(f^{-1}(V) \cap U)} {\Vol_{n}(f^{-1}(V))} \le \frac{\Vol_{n}(U)}{\Vol_{n}(f^{-1}(V))}$, 
	thus a lower bound of 
	$\Vol_{n}(f^{-1}(V))$ leads to an upper bound for $\fprec(U, V)$.
	Further note that
	\begin{align}
	\label{eq:largefiber}
		\Vol_{n}(f^{-1}(V)) 
		& = \Vol_{n}\left(f^{-1}(y \oplus r_{V})\right) \nonumber\\
		& \ge \Vol_{n}\left(f^{-1}(y) \oplus (r_{V}/L) \right) \nonumber\\
		& \ge  \frac{1}{2 \pi R}  \Vol_{n-m+1}(S^{n-m+1}) \Vol_{m}(B^{m}_1) R^{n-m+1} p^{m}(r_{V}/L) \nonumber\\
		& = \frac{\pi^{(n-m)/2}}{\Gamma(\frac{n-m}{2} + 1)} \frac{\pi^{m/2}}{\Gamma(\frac{m}{2} + 1)} R^{n-m} p^{m}(r_{V}/L)\,, 
	\end{align}
	where the first inequality is due to \cref{prop:lowerboundballvol}, the second inequality is due to the Waist Inequality \cref{eq:waistineq}, and $p^{m}(x) = x^m \left(1 + o(1)\right)$.
	Combining the volume calculation on $U$, 
	\begin{align*}
	\fprec(U, V) 
	& \le \frac{ \frac{\pi^{n/2}}{\Gamma(\frac{n}{2} + 1)} r_{U}^{n} }{ \frac{\pi^{n-m/2}}{\Gamma(\frac{n-m}{2} + 1)} \frac{\pi^{m/2}}{\Gamma(\frac{m}{2} + 1)} R^{n-m} p^{m}(r_{V}/L )} \\
	 & \le   \frac{\Gamma(\frac{n-m}{2} + 1) \Gamma(\frac{m}{2} + 1)}{\Gamma(\frac{n}{2} + 1)} (\frac{r_U}{R})^{n-m} \frac{r_U^{m}}{p^{m}(r_{V}/L)}.
	\end{align*} 
\end{proof}

\cref{thm:precision_bound_ball} generalizes as long as there is a corresponding waist theorem for that space. 
And roughly the condition of having a waist theorem is that a space is `truly' $n$ dimensional. 
We therefore conjecture that \cref{thm:precision_bound_ball} holds in various settings in machine learning where we are dealing with truly $n$ dimensional data. 
In the rest of this section, we are going to prove analogues of \cref{thm:precision_bound_ball} under the non-Euclidean norm. 

We define the necessary concepts first.
In the non-Eucldiean case, the generalized unit ball is a convex body. 
\begin{definition}[Generalized Unit Ball, e.g. \citet{wang2005volumes}]
\label{def:gen_ball}
	Let $p_1, p_2, \ldots, p_n \ge 1$. 
	A generalized unit $n$ ball is defined as the following convex body: 
	\begin{align} 
		B_{p_1, p_2, \ldots, p_n} 
		= 
		\{ ( x_1, x_2, \ldots, x_n ): 
		|x_1|^{p_1} + \ldots + |x_n|^{p_n} \leq 1 \} 
	\end{align}   
\end{definition}
\begin{theorem}[Volume of Generalized Ball, \citet{wang2005volumes}]
\label{thm:gen_ball}
	\begin{align}
		\Vol_n B_{p_1, p_2, \ldots, p_n} 
		= 
		2^n \frac{ \Gamma(1 + 1/p_1) \ldots \Gamma(1 + 1/p_n) }{\Gamma( 1 + 1/p_1 + \ldots + 1/p_n )}
	\end{align} 
\end{theorem}

\begin{definition}[Log-Concave Measure]
\label{def:log_concave}
	A Borel measure $\mu$ on $\mathbb{R}^n$ is log-concave if for any compacts sets $A \subset \mathbb{R}^n$ and $B \subset \mathbb{R}^n $, 
	and for any $0 < \lambda < 1$:
	\begin{align}
		\mu( \lambda A \oplus (1 - \lambda) B ) 
		\geq
		\mu(A)^{\lambda} \mu(B)^{1-\lambda} 
	\end{align}
\end{definition}

\begin{theorem}[Brunn-Minkowski Inequality]
	\label{thm:brunn_mink}
	Let $\Vol_n$ denote Lebesgue measure on $\mathbb{R}^n$. Let $A$ and $B$ be two nonempty compact subsets of $\mathbb{R}^n$. Then:
	\begin{align}
		[\Vol_n(A \oplus B)]^{1/n} \geq [\Vol_n(A)]^{1/n} + [\Vol_n(B)]^{1/n}
	\end{align}
\end{theorem}
The following lemma is well known in concentration of measure and convex geometry. 
We prove it here for completeness.

\begin{lem}[Lebesgue Measure on Convex Sets is Log-Concave]
\label{lem:log-concave}
	Let $\Vol_n$ denote Lebesgue measure on $\mathbb{R}^n$. 
	The (induced) restricted measure, $\Vol_n$, by restricting $\Vol_n$ to any convex sets is log-concave.
\end{lem}

\begin{proof}
	Plugging $\lambda A$ and $(1 - \lambda) B$ to theorem \ref{thm:brunn_mink}, 
	we have:
	\begin{align}
		\Vol_n^{1/n}(\lambda A \oplus (1-\lambda) B) 
		\geq 
		& \Vol_n^{1/n}(\lambda A) + \Vol_n^{1/n}( (1-\lambda) B) \\
		= 
		& \lambda \Vol_n^{1/n}(A) + (1 - \lambda) \Vol_n^{1/n}(B) \\
		\geq 
		& \Vol_n^{\lambda/n}(A) \Vol_n^{(1 - \lambda)/n}(B)
	\end{align}
	where the first equality follows because the $\lambda$ (or $ 1 - \lambda $ respectively) is scaled be a factor or $\lambda^n $ and taking $n$th root gives the equality, 
	and the last inequality follows from the weighted arithmetic-geometric mean inequality.
	Raising to the $n$th power, 
	we get:
	\begin{align}
		\Vol_n(\lambda A \oplus (1-\lambda) B) 
		\geq 
		\Vol_n^{\lambda}(A) \Vol_n^{ (1 - \lambda)}(B)
	\end{align}
	To finish the proof, 
	we note that for any $A$ and $B$ as nonempty compact subsets of a convex set $K \subset \mathbb{R}^n$ in the Euclidean space, 
	the Lebesgue measures restricted on $K$, 
	$\Vol_n(A)$ and $\Vol_n(B)$ can be written as Lebegues measures on $A$ and $B$.
	Convexity of $K$ ensures $\lambda A \oplus (1 - \lambda) B$ is still in the set $K$. 
\end{proof}
To deduce an analogue of \cref{thm:precision_bound_ball}, 
we need the following waist inequality for log-concave measures. 

\begin{theorem}[Waists of Arbitrary Norms, Theorem 5.4 of \citet{akopyan2018waist}]
\label{thm:waist_norm}
	Suppose $K \subset \mathbb{R}^n$ is a convex body, 
	$\mu$ a finite log-concave measure supported on $K$, 
	and $f: K \longrightarrow \mathbb{R}^m $ is continuous. 
	Then for any $\epsilon \in [0, 1]$ there exists $y \in \mathbb{R}^m $ such that:
	\begin{align}
		\mu( f^{-1}(y) \oplus \epsilon K ) \geq \epsilon^m \mu(K)
	\end{align}
\end{theorem}

\begin{prop}[Precision on Arbitrarilly Normed Balls]
\label{prp:precision_norm}
	Let $m < n$. 
	Let $f: B_{R; p_1, p_2, \ldots, p_n} \longrightarrow \mathbb{R}^m $ be a $L$-Lipschtiz continuous map defined on a generalized $n$ ball with radius $R$ from \cref{def:gen_ball}. 
	Let $r_U$ and $r_V$ be radii of two generalized balls, 
	with dimensions $n$ and $m$ respectively. 
	Then there exists $y$ depending on $r_V$ such that:
	\begin{align}
		Prec^f(U, V) \leq (\frac{r_U}{R})^{n-m} (\frac{r_U}{r_V/L})^{m}
	\end{align}
\end{prop}

\begin{proof}
	We would like to apply theorem \ref{thm:waist_norm}.
	Since $B_{R; p_1, p_2, \ldots, p_n}$ is a convex body, 
	Lebesgue meaure $\Vol_n$ on $B_{R; p_1, p_2, \ldots, p_n}$ is log-concave by \cref{lem:log-concave}. 
	Then by \cref{thm:waist_norm},
	for $r_V/L$, 
	there exists $y \in \mathbb{R}^m $ such that: 
	\begin{align}
		\Vol_n( f^{-1}(y) \oplus \frac{r_V}{L} K ) \geq (\frac{r_V}{L})^m \Vol_n(K)
	\end{align}
	where $K = B_{R; p_1, p_2, \ldots, p_n}$.
	Now by \cref{prop:lowerboundballvol}, 
	\begin{align}
		\Vol_n( f^{-1}( V ) ) 
		=
		\Vol_n( f^{-1}( B_{r_V; p_1, p_2, \ldots, p_n} ) ) 
		\geq 
		(\frac{r_V}{L})^m \Vol_n(K)
	\end{align}
	Therefore:  
	\begin{align}
		Prec^f(U, V) & \leq \frac{\Vol_n(U)}{\Vol_n(f^{-1}(V))} \\
		& \leq \frac{ \Vol_n( B_{r_U; p_1, p_2, \ldots, p_n} ) }{ \Vol_n(B_{R; p_1, p_2, \ldots, p_n}) (\frac{r_V}{L})^m } \\
		& = \frac{ 2^n \frac{ \Gamma(1 + 1/p_1) \ldots \Gamma(1 + 1/p_n) }{\Gamma( 1 + 1/p_1 +\ldots + 1/p_n )} r_U^{n-m} r_U^{m} }{ 2^n \frac{ \Gamma(1 + 1/p_1) \ldots \Gamma(1 + 1/p_n) }{\Gamma( 1 + 1/p_1 + \ldots + 1/p_n )} R^{n-m} (\frac{r_V}{L})^m } \\
		& = (\frac{r_U}{R})^{n-m} (\frac{r_U}{r_V/L})^{m}
	\end{align}
\end{proof}

\section{Proof of \cref{thm:precision_bound_avg}}
\label{sec:precision_avg_case}
The proof of \cref{thm:precision_bound_avg} is based on the idea that the fibers of certain type of continuous DR maps are mostly `large'.   
A map $f$ has a large fiber at $y$ if $f^{-1}(y)$'s volume is lower bounded by that of a linear map.  
This concept of `large' fiber is actually an essential concept in the proof of the waist inequality. 
The intuition we try to capture is that fibers of $f$ are considered big if their $n-m$ volumes are comparable to that of a surjective linear map.

The next two theorems show that for either of the following cases:
\begin{itemize}
	\item $m=1$; or
	\item $f: B_{R}^{n} \to \mathbb{R}^m$ be a $k$-layer neural network map with Lipschitz constant $L$, whose linear layers are surjective.
\end{itemize} 
the fibers of $f$ are mostly `large'.  

\begin{theorem}[Average Waist Inequality for Balls, m = 1]
	\label{thm:waistinequality_ball_avg}
	Let $f$ be a continuous map from $B^n_R$ to $\mathbb{R}$, and $\tau = \Vol_{n+1}\left(\Proj^{-1}(y) \oplus \eps\right) $ for an arbitrary $y \in \Proj(S^{n+1}_R) $, 
	then for all $\epsilon > 0$	
	\begin{align*}
	&\Vol_n \left(\left\{z\in B_R^{n}\,:\, \Vol_n\left( f^{-1}\left(f(z)\right) \oplus \eps \right) \ge \frac{1}{2\pi R} \tau  \right \} \right) \\
	\ge & \frac{1}{2\pi R}\Vol_{n+1} \left(
	\left\{ x \in S^{n+1}_R: 
	\Vol_{n+1}\left(\Proj^{-1}\left( \Proj(x) \right) \oplus \epsilon\right) \geq \tau
	\right\} \right).
	\end{align*}
\end{theorem}

\begin{prop}
	\label{prop:deep_net}
	Let $f$ be a $k$ layer neural network with nonlinear activations (ReLu, LeakyReLu, tanh, etc.) 
	from $B^n_R$ to $(0, 1)^m$ and $\Proj$ be an arbitrary linear projection on $B_R^{n}$.
		Then for any $\tau$ the following inequality holds,
	\begin{align*}
	&\Vol_n \left(\left\{x\in B_R^{n}\,:\, \Vol_n\left( f^{-1}\left(f(x)\right) \oplus \eps \right) 
	\ge
	\tau  \right \} \right) \\
	\ge 
	& \Vol_{n} \left(
	\left\{ x \in B^{n}_R: 
	\Vol_{n}\left(\Proj^{-1}\left( \Proj(x) \right) \oplus \epsilon\right) \geq \tau
	\right\} \right).
	\end{align*}
\end{prop}

The proof of \cref{thm:waistinequality_ball_avg} is postponed to \cref{appsubsec:proofm1}, while the proof of \cref{prop:deep_net} is postponed to \cref{appsubsec:proofdeepnet}.
We are now ready to derive a bound on DR maps' \textbf{average-case performance} over the domain based on \cref{thm:waistinequality_ball_avg} and \cref{prop:deep_net}. 
\begin{proof}[Proof of \cref{thm:precision_bound_avg}]
	We only present the proof when $f: B_{R}^{n} \to \mathbb{R}^m$ is a $k$-layer neural network map with Lipschitz constant $L$ by \cref{prop:deep_net}. The other case can be proved similarly by \cref{thm:waistinequality_ball_avg}.
	
	Given any $y\in \Proj(B_R^{n})$, pick $\tau = \Vol_{n}\left(\Proj^{-1}(y) \oplus \eps\right)$. By \cref{prop:deep_net} for all $\epsilon > 0$,
	\begin{align}
	\label{eqn:vol_n_version}
	&\Vol_n \left(\left\{x\in B_R^{n}\,:\, \Vol_n\left( f^{-1}\left(f(x)\right) \oplus \eps \right) 
	\ge
	\Vol_{n}\left(\Proj^{-1}(y) \oplus \eps\right)  \right \} \right) \\
	\ge 
	& \Vol_{n} \left(
	\left\{ x \in B^{n}_R: 
	\Vol_{n}\left(\Proj^{-1}\left( \Proj(x) \right) \oplus \epsilon\right) \geq \Vol_{n}\left(\Proj^{-1}(y) \oplus \eps\right)
	\right\} \right). \nonumber
	\end{align}
	Since $\Proj$ is a linear map, we have
				\begin{align*}
	& \Vol_{n} \left(
	\left\{ x \in B^{n}_R: 
	\Vol_{n}\left(\Proj^{-1}\left( \Proj(x) \right) \oplus \epsilon \right) 
	\geq 
	\Vol_{n}\left(\Proj^{-1}(y) \oplus \epsilon \right)
	\right\} \right) \\
	& = \Vol_{n} \left(
	\left\{ x \in B^{n}_R: 
	\Vol_{n-m}\left(\Proj^{-1}\left( \Proj(x) \right) \right) 
	\geq 
	\Vol_{n-m}\left(\Proj^{-1}(y) \right)
	\right\} \right).
	\end{align*}
	Further note that $\Proj^{-1}(y)$ is an $n-m$ ball with radius $r(y) = \sqrt{R^2 - \norm{y}^2}$. Thus,
	\begin{align*}
	& \Vol_{n} \left(
	\left\{ x \in B^{n}_R: 
	\Vol_{n-m}\left(\Proj^{-1}\left( \Proj(x) \right) \right) 
	\geq 
	\Vol_{n-m}\left(\Proj^{-1}(y) \right)
	\right\} \right) \\
	& = \int_{ B^m_{ \norm{y} } } \Vol_{n-m} (\Proj^{-1}(t)) \text{d}t.
	\end{align*}
	Therefore, 
	\begin{align}
	\label{eq:largefiberprob}
		\Vol_{n} \left( \left\{x\in B_R^{n}\,:\, \Vol_{n}\left( f^{-1}\left(f(x)\right) \oplus \epsilon \right) 
		\ge
		\Vol_{n}\left( \Proj^{-1}(y)  \oplus \epsilon \right)  \right \} \right)
		 \ge 
		 \int_{ B^m_{ \norm{y} } } \Vol_{n-m} (\Proj^{-1}(t)) \text{d}t. 
		 	\end{align}
	Lastly, pick $y$ such that $ \norm{y} = \sqrt{ R^2 - r_U^2 - \delta^2 } $, 
	so $ \Proj^{-1}(y) $ has radius $ \sqrt{r_U^2 + \delta^2} $.
	Let $\mathcal{E}$ denote the event
$
	\Vol_{n}\left( f^{-1}\left(f(x)\right) \oplus \epsilon \right) 
	\ge
	\Vol_{n}\left(B^{n-m}_{\sqrt{r_U^2 + \delta^2}} \oplus \epsilon \right)
$, thus
\[
\Prob(\mathcal{E}) \ge \frac{ \int_{ B^m_{ \sqrt{ R^2 - u^2 - \delta^2 } } } \Vol_{n-m} \Proj^{-1}(t) \text{d}t }{ \Vol_n(B^n_R) }.
\]

	The remaining proof is almost identical to the proof of \cref{thm:precision_bound_ball}.  
	Under the event $\mathcal{E}$, 
	\begin{align}
	\label{eq:largefiber_avg}
	\Vol_{n}(f^{-1}(V)) 
	& = \Vol_{n}\left(f^{-1}(y \oplus r_{V})\right) \nonumber\\
	& \ge \Vol_{n}\left(f^{-1}(y) \oplus (r_{V}/L) \right) \nonumber\\
	& \ge \Vol_{n-m}(B^{n-m}_1) \Vol_{m}(B^{m}_1) ( \sqrt{r_U^2 + \delta^2} )^{n-m} (r_{V}/L)^m \nonumber\\
	& = \frac{\pi^{(n-m)/2}}{\Gamma(\frac{n-m}{2} + 1)} \frac{\pi^{m/2}}{\Gamma(\frac{m}{2} + 1)} \sqrt{r_U^2 + \delta^2}^{n-m} (r_{V}/L)^{m} \,, 
	\end{align}
	where the first inequality is due to \cref{prop:lowerboundballvol}, 
	the second inequality is due to the event $\mathcal{E}$.
	Combining the volume calculation on $U$, 
	\begin{align*}
	\fprec(U, V) 
	& \le \frac{ \frac{\pi^{n/2}}{\Gamma(\frac{n}{2} + 1)} r_{U}^{n} }{ \frac{\pi^{n-m/2}}{\Gamma(\frac{n-m}{2} + 1)} \frac{\pi^{m/2}}{\Gamma(\frac{m}{2} + 1)} \sqrt{r_U^2 + \delta^2}^{n-m} (r_{V}/L )^{m}} \\
	& \le  \frac{\Gamma(\frac{n-m}{2} + 1) \Gamma(\frac{m}{2} + 1)}{\Gamma(\frac{n}{2} + 1)} (\frac{r_U}{\sqrt{r_U^2 + \delta^2}})^{n-m} \frac{r_U^{m}}{(r_{V}/L)^m}.
	\end{align*} 
	
\end{proof}

\subsection{Proof of \cref{thm:waistinequality_ball_avg}}
\label{appsubsec:proofm1}
The proof uses the following average waist inequality for spheres.
Let $P: S^{n+1}_R \longrightarrow B^n_R $ be the orthogonal projection, $\sigma_R $ and $\nu_R$ denote the corresponding Hausdorff measures on $S^{n+1}_R $ and $B^n_R$. 
Further, let $\Proj: S_R^{n+1} \rightarrow \real $ be the restriction to $S^{n+1}_R$ of a surjective linear map $ \widehat{\Proj}: \real^{n+2} \rightarrow \real$.
\begin{theorem}[Average Waist Inequality for Spheres \citepapp{alpert2015family}]
	\label{thm:waistinequality_sphere}
	Let $f$ be a continuous map from $S^{n+1}_R$ to $\real$, 
	then for all $y \in \Proj(S^{n+1}_R) $, we have: 
	\begin{align*}
	& \Vol_{n+1} 
	\{ x \in S^{n+1}_R: 
	\Vol_{n+1}(f^{-1}( f(x) ) \oplus \epsilon) 
	\geq 
	\Vol_{n+1}(\Proj^{-1}(y) \oplus \epsilon) 
	\} \\
	& \geq \\
	& \Vol_{n+1}
	\{ x \in S^{n+1}_R: 
	\Vol_{n+1}(\Proj^{-1}( \Proj(x) ) \oplus \epsilon) 
	\geq 
	\Vol_{n+1}(\Proj^{-1}(y) \oplus \epsilon) 
	\},
	\end{align*}
	where  
	\begin{equation*}
	\Vol_{n+1}\left(\Proj^{-1}(y)\oplus \eps \right) 
	=  
	2\pi \Vol_{n} \left(S_{ R_{\Proj^{-1}(y)} }^{n}\right) \Vol_{1} \left(B^{1}_1\right) \left(\ p^{1}(\eps) \right)\,,
	\end{equation*}
	$p^{1}(\eps)$ is $\eps \left(1 + o(1)\right)$, 
	i.e. 
	$ \displaystyle{ \lim_{\eps \to 0 } \frac{p^{1}(\eps)}{\eps} = 1 } $,
	and $ f^{-1}(y) \oplus \eps$ denotes the set of points $ x \in S^{n+1}_R$ such that $ d(x, f^{-1}(y)) < \eps$,
	$S_R^{n}$ is the $n$-dimensional sphere of radius $R$, 
	and $S_{ R_{\Proj^{-1}(y)} }^{n}$ is the sphere with radius $R_{\Proj^{-1}(y)}$ depending on where $y$ is taken in $f(S^{n+1}_R)$,
	i.e.
	$R_{\Proj^{-1}(y)}^2 = R^2 - y^2 $.
\end{theorem}
We are going to adapt the proof technique of theorem 1 from \citeapp{akopyan2017tight},
by replacing the existential waist inequality \eqref{thm:waistinequality} with its average version - theorem \ref{thm:waistinequality_sphere}. 
We need the following lemma: 

\begin{lem}[ Orthogonal Projection e.g. \citet{akopyan2017tight} ]
	\label{lem:ortho_projection} 
	Let $P: S^{n+1}_R \longrightarrow B^n_R $ be the orthogonal projection. 
	Then $P$ is $1$ - Lipschitz and 
	$P_{\#} \sigma_R =  2 \pi R \nu_R $.
	In other words, 
	$P$ sends the uniform Hausdorff measure $\sigma$ in $S^{n+1}_R$ to the uniform Lebesgue measure $\nu_n$ in $B^n_R$ up to constant $2 \pi R$. 
\end{lem}

\begin{proof}[Proof of \cref{thm:waistinequality_ball_avg}] 
	Given a map $f: B^n_R \longrightarrow \mathbb{R} $,
	consider $\hat{f} = f \circ P: S^{n+1}_R \rightarrow \real $, where $P$ is the orthogonal projection.
	By \cref{lem:ortho_projection}, $P$ is $1$-Lipschitz, thus for any $y \in \real$,
	\begin{equation}
	\label{eq:theorem9eq1}
	P^{-1}\left(f^{-1}(y)\right) \oplus \eps \subset P^{-1}\left(f^{-1}(y) \oplus \eps\right) \, \Rightarrow \Vol_{n+1}\left(\hat{f}^{-1}(y) \oplus \eps \right) \le \Vol_{n+1}\left( P^{-1}\left(f^{-1}(y) \oplus \eps\right) \right).
	\end{equation}
	Further, since $P_{\#} \sigma_R =  2 \pi R \nu_R $,
	\begin{equation}
	\label{eq:theorem9eq2}
	\Vol_{n+1}\left( P^{-1}\left(f^{-1}(y) \oplus \eps\right) \right) = 2\pi R\Vol_n\left( f^{-1}(y) \oplus \eps \right).
	\end{equation}
	Combining \cref{eq:theorem9eq1,eq:theorem9eq2}, for $\tau \in \real$,
	\begin{align}
	\label{eq:theorem9eq3}
	& \left\{ x\in S_R^{n+1}\,:\, \hat{f}(x) = y,  \Vol_{n+1}\left(\hat{f}^{-1}(y) \oplus \eps \right)\ge \tau \right\} \nonumber\\
	& \quad \subset  
	\left\{x\in S_R^{n+1}\,:\, \hat{f}(x) = y, \Vol_n\left( f^{-1}(y) \oplus \eps \right) \ge \frac{\tau}{2\pi R} \right\}.
	\end{align}
	Similarly, by $P_{\#} \sigma_R =  2 \pi R \nu_R $,
	\begin{align}
	\label{eq:theorem9eq4}
	& \Vol_{n+1} \left(\left\{x\in S_R^{n+1}\,:\, \hat{f}(x) = y, \Vol_n\left( f^{-1}(y) \oplus \eps \right) \ge \frac{\tau}{2\pi R} \right\} \right)\nonumber\\
	& \quad = 2\pi R \Vol_n \left(\left\{z\in B_R^{n}\,:\, f(z) = y, \Vol_n\left( f^{-1}(y) \oplus \eps \right) \ge \frac{\tau}{2\pi R}\right \} \right).
	\end{align}
	Thus by combining \cref{eq:theorem9eq3,eq:theorem9eq4}, we have 
	\begin{align*}
	& \Vol_{n+1} \left(\left\{ x\in S_R^{n+1}\,:\, \hat{f}(x) = y,  \Vol_{n+1}\left(\hat{f}^{-1}(y) \oplus \eps \right)\ge \tau \right\}  \right) \\
	\le &  \, 2\pi R \Vol_n \left(\left\{z\in B_R^{n}\,:\, f(z) = y, \Vol_n\left( f^{-1}(y) \oplus \eps \right) \ge \frac{\tau}{2\pi R}\right \} \right)
	\end{align*}
	Finally, note that $\hat{f}$ meets the condition in theorem \ref{thm:waistinequality_sphere}. 
	Thus for all $y \in \Proj(S^{n+1}_R)$: 
	\begin{align*}
	&\Vol_{n+1} \left(
	\left\{ x \in S^{n+1}_R: 
	\Vol_{n+1}\left(\Proj^{-1}\left( \Proj(x) \right) \oplus \epsilon\right) 
	\geq 
	\Vol_{n+1}\left(\Proj^{-1}(y) \oplus \eps\right) 
	\right\} \right) \\
	\le & \Vol_{n+1} \left(
	\left\{ x \in S^{n+1}_R: 
	\Vol_{n+1}\left( \hat{f}^{-1} \left(\hat{f}(x) \right)\oplus \eps\right) 
	\geq 
	\Vol_{n+1}\left(\Proj^{-1}(y) \oplus \eps\right) 
	\right\} \right) \\
	\le & 2\pi R \Vol_n \left(\left\{z\in B_R^{n}\,:\, \Vol_n\left( f^{-1}\left(f(z)\right) \oplus \eps \right) \ge \frac{1}{2\pi R} \Vol_{n+1}\left(\Proj^{-1}(y) \oplus \eps\right)  \right \} \right).
	\end{align*}
	
\end{proof}

\subsection{Proof of \cref{prop:deep_net}}
\label{appsubsec:proofdeepnet}
We first prove that \cref{prop:deep_net} holds for any surjective linear map.
\begin{prop}
	\label{prp:pca}
	Let $f$ be any surjective linear map (PCA, linear neural networks) from $B^n_R$ to $\mathbb{R}^m$, and $\Proj$ be an arbitrary surjective linear projection from $B_R^{n}$ to $\mathbb{R}^m$. 
		Then for any $\tau$ the following inequality holds,
	\begin{align*}
	&\Vol_n \left(\left\{x\in B_R^{n}\,:\, \Vol_n\left( f^{-1}\left(f(x)\right) \oplus \eps \right) 
	\ge
	\tau  \right \} \right) \\
	\ge 
	& \Vol_{n} \left(
	\left\{ x \in B^{n}_R: 
	\Vol_{n}\left(\Proj^{-1}\left( \Proj(x) \right) \oplus \epsilon\right) \geq \tau
	\right\} \right).
	\end{align*}
\end{prop}
\begin{proof}
	By the singular value decomposition,
	any linear dimension reduction map $f$ can be decomposed as a composition or unitary operators ($U_m$ and $V_n$), 
	signed dialation of full rank
	($\Sigma$),
	and projection operator of rank $m$
	($\widehat{\Proj}$), 
	where $\widehat{\Proj}$ linearly projects from $\mathbb{R}^n$ to $\mathbb{R}^m$ 
	(or more commonly $\Sigma \circ \widehat{\Proj} $ is called rectangular diagonal matrix map):  
	$f = U_m \circ \Sigma \circ \widehat{\Proj} \circ V^{*}_n $. 
	The set
	\begin{align*}
	&\left\{x\in B_R^{n}\,:\, \Vol_n\left( f^{-1}\left(f(x)\right) \oplus \eps \right)
	\ge
	\tau \right \} \\  
	=& 
	\left\{x\in B_R^{n}\,:\, \Vol_n\left( 
	(U_m \circ \Sigma \circ \widehat{\Proj} \circ V^{*}_n)^{-1} \left( U_m \circ \Sigma \circ \widehat{\Proj} \circ V^{*}_n (x)\right) \oplus \eps \right)
	\ge
	\tau \right \} \\
	= & 
	\left\{x\in B_R^{n}\,:\, \Vol_n\left( 
	(V^{*}_n)^{-1} \circ \widehat{\Proj}^{-1} \circ \Sigma^{-1} \circ U_m^{-1} \circ U_m \circ \Sigma \left( \widehat{\Proj} \circ V^{*}_n (x)\right) \oplus \eps \right)
	\ge
	\tau \right \} \\
	= &
	\left\{x\in B_R^{n}\,:\, \Vol_n\left( 
	(V^{*}_n)^{-1} \circ \widehat{\Proj}^{-1} \circ \left( \widehat{\Proj} \circ V^{*}_n (x)\right) \oplus \eps \right)
	\ge
	\tau \right \} \\
	= &
	\left\{x\in B_R^{n}\,:\, \Vol_n\left( 
	V_n \circ \widehat{\Proj}^{-1} \circ \left( \widehat{\Proj} \circ V^{*}_n (x)\right) \oplus \eps \right)
	\ge
	\tau \right \} \\ 
	= &
	\left\{x\in B_R^{n}\,:\, \Vol_n\left( 
	V_n \circ \widehat{\Proj}^{-1} \circ \left( \widehat{\Proj} (x)\right) \oplus \eps \right)
	\ge
	\tau \right \}  \\
	= &
	\left\{x\in B_R^{n}\,:\, \Vol_n\left( 
	\widehat{\Proj}^{-1} \circ \widehat{\Proj} (x) \oplus \eps \right)
	\ge
	\tau \right \}, 
	\end{align*} 
	where the last two equalities follow because unitary operator $V^{*}_n$ and $V_n$ don't affect volumes because they are linear isometries. 
	We note this shows the distribution of fiber volume is the same for any surjective linear map. 
	Finally, note that by symmetry, 	
	\[
	\left\{x\in B_R^{n}\,:\, \Vol_n\left( 
	\widehat{\Proj}^{-1} \circ \widehat{\Proj} (x) \oplus \eps \right)
	\ge
	\tau \right \}  
	=
	\left\{x\in B_R^{n}\,:\, \Vol_n\left( 
	\Proj^{-1} \circ \Proj (x) \oplus \eps \right)
	\ge
	\tau \right \} . 
	\]
									
\end{proof}

\begin{lem}[Monotonicity of Fiber Volume under Compositions]
	\label{lem:composition_fiber}
	Let $f: B^n_R \longrightarrow X$ and $g: X \longrightarrow \mathbb{R}^m $ be any maps for some set $X$. 
	Then for any $\tau$ we have the following inequality:
	\begin{align*}
	&\Vol_n \left(\left\{x\in B_R^{n}\,:\, \Vol_n\left( (f \circ g) ^{-1}\left(f \circ g\right(x)) \oplus \eps \right) 
	\ge
	\tau  \right \} \right) \\
	\ge 
	& \Vol_{n} \left(
	\left\{ x \in B^{n}_R: 
	\Vol_{n}\left(f^{-1}\left( f(x) \right) \oplus \epsilon\right) \geq \tau
	\right\} \right).
	\end{align*}
\end{lem}

\begin{proof}
	Consider: 
	$a \in \left\{ x \in B^{n}_R: 
	\Vol_{n}\left(f^{-1}\left( f(x) \right) \oplus \epsilon\right) \geq \tau
	\right\} $ and we let $b = f(a)$.
	We obviously have $b \in g^{-1} \circ g(b)$. Therefore $a \in f^{-1}(b) \subset f^{-1} \circ g^{-1} \circ g( f(a) )$. 
	Thus,
	\begin{align*}
	& \left\{ x \in B^{n}_R: 
	\Vol_{n}\left(f^{-1}\left( f(x) \right) \oplus \epsilon\right) \geq \tau
	\right\}
	\subset 
	& \left\{x\in B_R^{n}\,:\, \Vol_n\left( (f \circ g) ^{-1}\left(f \circ g\right(x)) \oplus \eps \right) 
	\ge
	\tau  \right \}.
	\end{align*}
\end{proof}

\begin{proof}[Proof of \cref{prop:deep_net}]
	We proceed by induction on $k$. 
	When $k = 1$, 
	it is given by lemma \ref{lem:composition_fiber},
	by noting a one layer net is a composition of any activation with a surjective linear map, $L_1$. 
	Assume this is true for a $k$ layer neural net, 
	$ f_k $, 
	with $k$ layers such that $k \geq 1$. 
	So we have: 
	\begin{align*}
	&\Vol_n \left(\left\{x\in B_R^{n}\,:\, \Vol_n\left( f_{k}^{-1}\left(f_{k}(x)\right) \oplus \eps \right) 
	\ge
	\tau  \right \} \right) \\
	\ge 
	& \Vol_{n} \left(
	\left\{ x \in B^{n}_R: 
	\Vol_{n}\left(\Proj^{-1}\left( \Proj(x) \right) \oplus \epsilon\right) \geq \tau
	\right\} \right).
	\end{align*}
	We need to check a neural net $f_{k+1}$ with $k+1$ layers:  
	$ f_{k+1} = \tanh \circ L_{k+1} \circ f_{k} $. 
	But this is again a composition between functions and we can apply \cref{lem:composition_fiber}. 
	This completes the proof. 
\end{proof}

In light of \cref{prp:pca}, we can characterize $\Proj_1^{-1}(t)$ and $\Proj_2^{-1}(t)$ explicitly. 
Since the bound holds for any surjective linear map, 
we can choose in particular $\Proj_1^{-1}(t)$ and $\Proj_2^{-1}(t)$ to be the coordinate projection from $\mathbb{R}^n$ to $\mathbb{R}^m$ (with all eigenvalues equal to 1). 
Then $t = ( t_1, \cdots, t_m ) \in B^{m}_R$, 
$\Proj_1^{-1}(t) = S^{n-m+1}_{\Re_1} $ and $\Proj_2^{-1}(t) = B^{n-m}_{\Re_2} $, where $\Re_1 = \Re_2 = \sqrt{ R^2 - \sum_{i=1}^m t_i^2 } $.

\section{Proofs for \cref{sec:wassersteinmeasure}}
\label{sec:W2proofs}
This section is devoted to the proofs for \cref{sec:wassersteinmeasure}.
We first present the proof of \cref{thm:wass_lowerbound}.
\begin{proof}[Proof of \cref{thm:wass_lowerbound}]
By \cref{eq:largefiber}, 
\[
{\Vol_{n}(f^{-1}(V))} 
\geq
\frac{\pi^{n/2}}{\Gamma(\frac{n-m}{2} + 1) \Gamma(\frac{m}{2} + 1)} R^{n-m} p^{m}(r_{V}/C).
\]
Let $B_{r^{\#}}$ be the ball with the same volume as ${\Vol_{n}(f^{-1}(V))} $ and a common center with $U$. 
Thus 
\begin{align}
\label{eq:rpoundlowerbound}
r^{\#} \ge r = 
(\frac{\Gamma(\frac{n}{2} + 1)}{ \Gamma(\frac{n-m}{2} + 1) \Gamma(\frac{m}{2} + 1) })^{ \frac	{1}{n} } R^{\frac{n-m}{n}} (p^{m}(r_{V}/C))^{ \frac{1}{n} }.
\end{align}
By \cref{thm:optimalball},
\[
W_{2}^2(\mathbb{P}_U, \mathbb{P}_{f^{-1}(V)}) 
\geq
W_{2}^2(\mathbb{P}_U, \mathbb{P}_{B_{r^{\#}}})
= \int_{\ball{r}{u}} | x - T(x) |^2 \,\text{d} \Prob_{B_{r^{\#}}} (x),
\]
thus it is sufficient to lower bound the last term. 
Under the condition that $\Vol_n(f^{-1}(V))  \ge \Vol_n(U)$,
\begin{align*}
\int_{B_{r^{\#}}} | x - T(x) |^2 \,\text{d} \Prob_{B_{r^{\#}}} (x)
 & = \int_{B_{r^{\#}}} | x - \frac{r_U}{r^{\#}}x |^2 \,\text{d} \Prob_{B_{r^{\#}}} (x) \\
 & =  \left( 1 - \frac{r_U}{r^{\#}} \right)^2  \int_{B_{r^{\#}}} |x |^2 \, \td \Prob_{B_{r^{\#}}} (x). \\
 \end{align*}
 Further, 
 \begin{align*}
 	\quad \int_{B_{r^{\#}}} |x |^2 \, \td \Prob_{B_{r^{\#}}} (x) 
  	& = 
  	\int_{0}^{r^{\#}} r^2 \frac{1}{ \Vol_{n}(f^{-1}(V)) } \td S^{n-1}(r) \td r \label{eq:thm11_1} \\ 
  	& =  \frac{1}{ \Vol_{n}(f^{-1}(V)) } \frac{2 \pi^{n/2} }{ \Gamma(\frac{n}{2}) }
  	\int_{0}^{r^{\#}} r^{n+1} \td r \\
		& = \frac{n}{n+2} (r^{\#})^2.
\end{align*}
Therefore, 
\[
W_{2}^2(\mathbb{P}_U, \mathbb{P}_{f^{-1}(V)}) \ge  \left( 1 - \frac{r_U}{r^{\#}} \right)^2 \frac{n}{n+2} (r^{\#})^2 = \frac{n}{n+2} (r^{\#} - r_U)^2.
\]
Note that the above lower bound is monotonically increasing with respect to $r^{\#}$ for $r^{\#} > r_U$. Therefore from \cref{eq:rpoundlowerbound}, when $r>r_U$, replacing $r^{\#}$ by $r$ gives a lower bound for $ W_{2}^2(\mathbb{P}_U, \mathbb{P}_{f^{-1}(V)}) $.

Further, 
note that as $n \rightarrow \infty$, $r \rightarrow R$, 
we have: 
\[
W_{2}^2(\mathbb{P}_U, \mathbb{P}_{f^{-1}(V)})  = \Omega\left((R - r_U)^2 \right).
\]
\end{proof}

The rest of this section is to prove \cref{thm:optimalball}. 
The key step is to show the following lemma.
\begin{lem}[Reduction to Optimal Partial Transport]
	\label{lem:symmetricuniformoptimizer}
	Given $f(x) = 1/ \cV \le 1/ \Vol(B_r)$, the optimal distribution $f_M$ for the optimal transport problem
	\begin{equation}
	\label{eq:symmetricuniformoptimizer}
	\min_{\Prob \,:\, \Prob \text{ is dominated by } f} W_2(\Prob, \Prob_{B_r}) 
	\end{equation}
	is the uniform distribution over $B_{r^{\#}}$ where $r^{\#}$ is the radius such that $\Vol(B_{r^{\#}}) = \cV$.
\end{lem}
By \cref{lem:symmetricuniformoptimizer}, let $f(x) = 1/ \cV $,  the optimal solution for the problem 
\[
\inf_{W:\,\Vol_n(W) = \cV } W_{2}(\mathbb{P}_U, \mathbb{P}_{W}) = W_{2}(\mathbb{P}_U, \mathbb{P}_{B_r})
\]
is the same as support of the optimizer of \cref{eq:symmetricuniformoptimizer}, 
thus proving the first statement of \cref{thm:optimalball}.

The proof of \cref{lem:symmetricuniformoptimizer} is based on the uniqueness of the optimal transport map for the optimal partial transport problem \citepapp{caffarelli2010free, figalli2010optimal}. 
We summarize the statements in \citepapp{figalli2010optimal}\footnote{The Brenier theorem is not stated in the paper, but it holds under standard derivation.} as a theorem here for completeness.

\begin{theorem}[\citetapp{figalli2010optimal}]
\label{thm:exist_unique}
Let $f, g \in L^1(B_{R}^{n})$ be two nonnegative functions, 
and denote by $\Xi_{\leq}(f, g)$ the set of nonnegative finite Borel measures on $B_{R}^{n} \times B_{R}^{n}$ whose first and second marginals are dominated by $f$ and $g$ respectively, 
i.e.
$\xi( A \times B_{R}^{n} ) \leq \int_{A} f(x) dx $ 
and 
$\xi( B_{R}^{n} \times A ) \leq \int_{A} g(y) dy $,  
for all Borel $A \subset B_{R}^{n} $. 
Denote $ \mathscr{M}(\xi) := \int_{ B_{R}^{n} \times B_{R}^{n} } d \xi $ 
and fix $M \in [ \lVert\min(f(x), g(x))\rVert_{L_1}, \min( \lVert f\rVert_{L_1}, \lVert g\rVert_{L_1} )]$.
Then there exists a unique optimizer $\xi_M$\footnote{up to a measure zero set} to the following optimal partial transport problem:
$$\inf_{\xi \in \Xi_{\leq}(f, g); \mathscr{M}(\xi) = M } C(\xi) 
= 
\inf_{\xi \in \Xi_{\leq}(f, g); \mathscr{M}(\xi) = M } \int_{B_{R}^{n} \times B_{R}^{n}} |x - y|^2 \text{d} \xi(x, y)
$$

Moreover, there exist Borel sets $A_1, A_2 \subset B_R^n$ such that 
$\xi_M$ has left and right marginals
whose densities $f_M=1_{A_1} f$ and $g_M=1_{A_2} g$ are given by the 
restrictions of $f$ and $g$ to $A_1$ and $A_2$ respectively, where $1_{A}$ denotes characteristic function on the set $A$. 

Finally, there exists a unique optimal transport map $T$\footnote{up to a measure zero set}, such that
$$\min_{\xi \in \Xi_{\leq}(f, g); \mathscr{M}(\xi) = M } 
C(\xi) 
= 
\int_{B_{R}^{n}} | T(x) - x|^2 \text{d}f_M(x), $$
where $f_M$ is the marginal of $\xi_M$ over the first $B_R^n$. 
\end{theorem}

We will prove \cref{lem:symmetricuniformoptimizer} in two different ways. 
The first is based on calculus and reducing the problem to one dimensional optimal transport. 
The second one utilizes the extreme points property that characterizes the densities $f_M=1_{A_1} f$ and $g_M=1_{A_2} g$ (Proposition 3.2 and Theorem 3.3 in \citep{korman2013insights}). 
\footnote{Such property can also be deduced from earlier work, e.g. Theorem 4.3 and Corollary 2.11 from \citep{caffarelli2010free}. 
But \citep{korman2013insights} is perhaps more direct and accessible.}

\begin{proof}[Proof of \cref{lem:symmetricuniformoptimizer}, first approach]
Let $\cV = \Vol(W)$, define $f(x) = 1/\cV$ be a constant function on $B_R^n$ and $g(x) = \frac{1}{\Vol(B_r)}$ if $x\in B_r$ and $0$ otherwise. Also, 
let $M = 1$. solving the problem 
\begin{equation}
\label{eq:lemw2prob}
\min_{\Prob \,:\, \Prob \text{ is dominated by } f} W_2(\Prob, \Prob_{B_r}) 
\end{equation}
is equivalent to solving the following optimal partial transport problem
\begin{equation}
\label{eq:partialoptimaltransprob}
\inf_{\xi \in \Xi_{\leq}(f, g); \mathscr{M}(\xi) = 1} C(\xi) = \inf_{\xi \in \Xi_{\leq}(f, g); \mathscr{M}(\xi) = 1} \int_{B_{R}^{n} \times B_{R}^{n}} |x - y|^2 \text{d} \xi(x, y).
\end{equation}
In particular, since $\Vol(B_R^n) \ge \cV  > \Vol(B_r)$,  it is straightforward to see that $\|\min(f(x), g(x))\|_{L_1}  = \Vol(B_r)/ \cV< 1 $, and  $\min\left(\|f(x)\|_{L_1},  \|g(x)\|_{L_1} \right) \ge 1$. 
By \cref{thm:exist_unique}, the optimization problem 
$
\inf_{\xi \in \Xi_{\leq}(f, g); \mathscr{M}(\xi) = 1} C(\xi)
$
has a unique solution $\xi^*$. 
Now given $\xi^*$, the optimal solution $\Prob^*$ of \cref{eq:lemw2prob} and $\Prob_{B_r}$ are the first and the second marginals of $\xi^*$. 
Thus it is sufficient to prove that the first marginal of $\xi^*$ is a uniform distribution.

Let $f_M$ be the first marginal of $\xi^*$ and $g_M = g$ be the second marginal. We first show that $f_M$ is rotationally invariant. To see that, for any rotation map $\rotate$, note that $\rotate(B_R^n) = B_R^n$, $\rotate(B_r) = B_r$, $f\circ\rotate = f$, and $g\circ\rotate = g$. Therefore, $f_M\circ \rotate$ is the unique optimal solution for the optimization problem 
\[
\inf_{\xi \in \Xi_{\leq}(f\circ\rotate, g\circ\rotate); \mathscr{M}(\xi) = 1} \int_{\rotate(B_R^n) \times \rotate(B_R^n)} |x - y|^2 \text{d} \xi(x, y) = \inf_{\xi \in \Xi_{\leq}(f, g); \mathscr{M}(\xi) = 1} \int_{B_{R}^{n} \times B_{R}^{n}} |x - y|^2 \text{d} \xi(x, y).
\]
Thus, $ f_M\circ \rotate = f_M$, i.e. $f_M$ is rotationally invariant, up to a measure zero set. For a density function to be rotationally invariant, it is straightforward that its support $S$ is also rotationally invariant, thus is a union of $(n-1)$ spheres. Similarly, one can also prove that $T$ is equivariant under rotations. 

We next prove that $f_M$ is a uniform distribution. 
Note that $g_M$ is a uniform distribution over $B_r$.
Define $\hat{G}(t)$ to be the the cumulative distribution $\hat{g}$ for $g_M$ in the polar coordinate marginalized on the sphere, i.e.,
\[
\hat{G}(t) = \int_{0}^{t} \frac{1}{ \Vol_n B^n_r } \Vol_{n-1} (S^{n-1}_u) \td u, 
\]
for every $ 0 \leq t \le r $, and $G(t) = 1$ for $t > r$. 
Similarly, since $f_M$ is also  rotationally invariant, we can also define its cumulative distribution in the polar coordinate marginalized on the sphere.
Note that $ \td \mu_{f_M} = f_M(x) \td S_r^{n-1} \td r$, let $\hat{f} (r) = \int f_M(x) \td S_r^{n-1} $, thus 
\[
F(B_t) =  \int_{B_t} f_M(x) \td S_u^{n-1} \td u= \int_{0}^t \int f_M(x) \td S_u^{n-1} \td u =  \int_{0}^t\hat{f}(u) \td u =  \hat{F}(t).
\]
Finally, note that $T$ is also rotationally invariant, thus $ W_2(f_M, \Prob_{B_r}) = W_2(\hat{f}, \hat{g})$.
It is sufficient to prove that $\hat{f} (u) = \Vol_{n-1}(S_u^{n-1})/\cV $ , thus by rotationally invariant $f_M(x) = 1/\cV$ is a uniform distribution.

Note that $\hat{F}(t) \le \hat{G}(t) $ and $\hat{f} (u) = \int f_M(x) \td S_u^{n-1}  \le \int f(x) \td S_u^{n-1}  = \Vol_{n-1}(S_u^{n-1})/\cV$. By a reformulation of the one dimensional Wasserstein distance \citepapp{vallender1974calculation}:
\begin{align}
\label{eq:opttransonedimension}
W_2(\hat{f}, \hat{g})= & \int_{0}^{1} | \hat{F}^{-1}(t) - \hat{G}^{-1}(t) |^2 \td t \nonumber\\
 = & \int_0^{r^\#} |x - \hat{G}^{-1}\left( \hat{F}(x)\right) |^2  \td \hat{F}(x),
\end{align}
which is just the area between between the graphs of $ \hat{F}(r) $ and $\hat{G}(r) $.
It is straightforward that the optimal $\hat{f}$ will maximize the growth rate of $\hat{F}$ in order to minimize the area, i.e. $\hat{f} (u) = \int f(x) \td S_u^{n-1} = 1/\cV \Vol_{n-1}(S_u^{n-1})$.
Therefore, $f_M(x) = 1/\cV$ is a uniform distribution over  $B_{r^{\#}}$ where $r^{\#}$ is the radius of $B_{r^{\#}}$ such that $\Vol(B_{r^{\#}}) = \cV$.
\end{proof}

\begin{proof}[Proof of \cref{lem:symmetricuniformoptimizer}, second approach]
The proof starts in exactly the same way as in the first approach, up to the rotational invariance part. 
Instead of using the polar coordinate argument, we directly apply by invoking the second statement in \cref{thm:exist_unique}, 
so $f_M=1_{A_1} f$. 
But we know that $f(x) = 1/\cV$ is a uniform distribution, and the claim follows. 
\end{proof}

Further note that by \cref{eq:opttransonedimension}, the optimal transport from $\hat{F}$ to $\hat{G}$ is 
\[
\hat{T}(u) = \hat{G}^{-1}\left( \hat{F}(u)\right) = \hat{G}^{-1}\left( \frac{1}{\cV} \Vol_n(B_u^n)\right) = \left(\frac{\Vol_n(B_{r_U})}{\cV}\right)^{1/n} u = \frac{r_U}{r_{\cV}}r,
\]
for $0\le r\le r_M$.
Note that $T$ is rotationally symmetric, thus the optimal transport $T(x) =\frac{r_U}{r_{\cV}}  x $, for $x\in B_{r_{\cV}}$

Lastly, it remains to prove
\[
\inf_{W:\,\Vol_n(W) \ge \cV} W_{2}(\mathbb{P}_U, \mathbb{P}_{W}) = \inf_{W:\,\Vol_n(W) = \cV} W_{2}(\mathbb{P}_U, \mathbb{P}_{W}), 
\]
which follows the next lemma.
\begin{lem}[Monotonicity of Volume Comparison]
	\label{prp:monotonicity}
	Given two balls $B_{r_1}$ and $B_{r_2}$ such that $\Vol(B_{r_1}) \ge \Vol(B_{r_2})$, then for any $A\subset \real^n$ such that $\Vol(A)\ge \Vol(B_{r_1})$,
	\[
	W_2(\mathbb{P}(A), \mathbb{P}(B_{r_2})) 
	\geq 
	W_2(\mathbb{P}(B_{r_1}), \mathbb{P}(B_{r_2})).
		\]
\end{lem}
\begin{proof}[Proof of 	\cref{prp:monotonicity}]
	We have shown that $W_2(\mathbb{P}(A), \mathbb{P}(B_{r_2})) \geq W_2(\mathbb{P}(B_{r_A}), \mathbb{P}(B_{r_2}))$, where $B_{r_A}$ is a ball with Volume $\Vol(A)$.
	It remains to prove that 
	\[
	W_2(\mathbb{P}(B_{r_A}), \mathbb{P}(B_{r_2}))	\ge W_2(\mathbb{P}(B_{r_1}), \mathbb{P}(B_{r_2}))
	\]  
		Let $T_A(x) = \frac{r_2}{r_A}x$, and $T_1(x) =  \frac{r_2}{r_1} x$.
	By \cref{thm:exist_unique},
	\begin{align*}
	W_2^2(\mathbb{P}(B^n_{r_{A}}), \mathbb{P}(B^n_{r_2})) 
	& = 
	\int_{B^n_R} |x - T_A(x)|^2 \td \mathbb{P}_{B^n_{r_{A}}} \\
	& =  	\int_{B^n_R} \left|x - \frac{r_2}{r_A} x\right|^2 \td \mathbb{P}_{B^n_{r_{A}}} \\
	& = \int_{B^n_R} \left(1 - \frac{r_2}{r_A} \right)^2\left|x\right|^2 \td \mathbb{P}_{B^n_{r_{A}}} \\
	& \ge \int_{B^n_R} \left(1 -\frac{r_2}{r_1}\right)^2\left|x\right|^2 \td \mathbb{P}_{B^n_{r_{A}}} \\
	& = \int_{B^n_R} |x - T_1(x)|^2 \td \mathbb{P}_{B^n_{r_{A}}} \\
	& = W_2^2(\mathbb{P}(B^n_{r_{1}}), \mathbb{P}(B^n_{r_2})) 
	\end{align*}	
\end{proof}

To make \cref{thm:optimalball} complete, it remains to investigate the remaining cases when $ 0 < \cV < \Vol_n(U)$. 
\begin{proof}
We claim that when $0 < \cV < \Vol(U)$, $\inf_{W:\,\Vol_n(W) = \cV} W_{2}(\mathbb{P}_U, \mathbb{P}_{W}) = 0$, and it is not attained by any set. 
Let $ \Vol_n(W_k) = \cV $ and keep $W_k \subset U $ such that the mass of $W_k$ is evenly distributed among the intersection between successively finer rectangular grids and $U$. 
Inside each intersection, the two distributions have the same probability mass. 
Since both are uniform probability distributions, their densities scale inversely proportional to their support sizes inside the intersection.  
Each little intersection is inside a little cube with width $\frac{2R}{k}$.  
We take $\xi$ to be the product measure between $\mathbb{P}(U)$ and $\mathbb{P}(W)$. 
Now, when we compute: 
\[
	W_{2}(\mathbb{P}_{U}, \mathbb{P}_{W}) 
	= 
	\inf_{\xi \in \Xi(\mathbb{P}_{U}, \mathbb{P}_{W})} \mathbb{E}_{(a, b) \sim \xi} [ \| a - b \|^{2}_{2} ]^{1/2}
	\le 
	\mathbb{E}_{(a, b) \sim \xi} [ \| a - b \|^{2}_{2} ]^{1/2}.
\]
The integrand $\| a - b \|^{2}_{2} \le \sqrt{n} \frac{2R}{k} $. 
By letting $k \rightarrow \infty$ (finer grids), we see that $\inf_{W:\,\Vol_n(W) = \cV} W_{2}(\mathbb{P}_U, \mathbb{P}_{W}) = 0$. 

However, the infimum is not attained by any set $W$ with $\Vol_n(W) = \cV < \Vol_n(U)$. 
Without loss of generality, we assume $W \subset U $. 
Then $ \Vol_n (U - W) > 0 $. So $W_{2}(\mathbb{P}_{U}, \mathbb{P}_{W}) > 0$.

\end{proof}

\section{Proofs for \cref{sec:metric_manifold}}
We prove the proposition \ref{prop:manifoldlearning} here. 
We begin with a lemma.
\begin{lem}[One-To-One $\implies$ Perfect Precision]
	\label{lem:open_map}
	Let $\mathcal{M} $ be a Riemannian manifold. 
	Let $f: \mathcal{M} \rightarrow \mathbb{R}^m $ be an open map. 
	Then $f$ achieves perfect precision. 
\end{lem}
\begin{proof}
	$f$ is an open map,
	mapping open sets to open sets.
	For every $ U \subset \mathcal{M} $,
	$ f(U) $ is open in $ \mathbb{R}^m $. 
	Since $f(U)$ is open and contains $y = f(x)$,
	there exists $r_V > 0$ such that $V \subset f(U) $.
	This implies $f^{-1}(V) \subset U $. 
	But then $\text{Precision}^{f} (U, V) = \frac{\Vol_{n}(f^{-1}(V) \cap U)} {\Vol_{n}(f^{-1}(V))} = 1$ for such $V$ and $U$. 
\end{proof}

\begin{proof}[Proof of Proposition \ref{prop:manifoldlearning}]
	Let $\mathcal{M} $ be an $n$-dimensional Riemannian manifold and $ m \geq 2n $ be the embedding dimension. 
	By the Whitney embedding theorem, 
	there exists a smooth map $f$ such that $ f( \mathcal{M} ) $ embeds into $ \mathbb{R}^m $. 
	Thus $f$ is an open map from $ \mathcal{M} $ to $ f( \mathcal{M} ) $. 
	We now apply lemma \ref{lem:open_map} to arrive at the conclusion. 
\end{proof}

\vspace{\presubsection}
\section{Wasserstein many-to-one, discontinuity and cost}
\vspace{\postsubsection}
\label{sec:wasserstein_experiments}
In general, we do not have theoretical lower bound for $W_2$ measure. 
It is natural to use the sample based Wassertein distances as substitutes. We perform some preliminary study of this heuristics below.

Recall Wasserstein distance is the minimal cost for mass-preserving transportation between regions.
The Wasserstein $L^2$ distance is:
\begin{equation}
W_{2}(\mathbb{P}_{a}, \mathbb{P}_{b}) = \inf_{\xi \in \Xi(\mathbb{P}_{a}, \mathbb{P}_{b})} \mathbb{E}_{(a, b) \sim \xi} [ \| a - b \|^{2}_{2} ]^{1/2}
\end{equation}
where $ \Xi(\mathbb{P}_{a}, \mathbb{P}_{b})$
denotes all joint distributions $\xi(a, b)$ whose marginal distributions are $\mathbb{P}_{a}$ and $\mathbb{P}_{b}$.
Intuitively, among all possible ways of transporting the two distributions,
it looks for the most efficient one.
With the same intuition,
we use Wasserstein distance between $U$ and $f^{-1}(V)$\footnote{The regions $U$ and $f^{-1}(V)$ are given uniform distribution,
i.e.
their densities are
$\frac{1}{Vol_{n}(U)}$ and $\frac{1}{Vol_{n}(f^{-1}(V))}$ }
~to measure precision (See \cref{sec:simulation}).
This not only captures similar overlapping information as the setwise precision: $\frac{Vol_{n}(f^{-1}(V) \cap U)} {Vol_{n}(f^{-1}(V))}$,
but also captures the shape differences and distances between $U$ and $f^{-1}(V)$.
Similarly,
Wasserstein distance between $f(U)$ and $V$ may capture the degree of discontinuity.
$W_{2}(\mathbb{P}_{f(U)}, \mathbb{P}_V)$
\textbf{captures continuity} and
$W_{2}(\mathbb{P}_U, \mathbb{P}_{f^{-1}(V)})$ \textbf{captures injectivity}.

In practice,
we calculate Wasserstein distances between two groups of samples,
$\{a_i\}$ and $\{b_j\}$,
using algorithms from \cite{bonneel2011displacement} .
Specifically,
we solve
\begin{align}
\begin{gathered}
\min_m \sum_i \sum_j d_{i, j} m_{i \rightarrow j}~,\\
\mathrm{such~that:}~~ m_{i \rightarrow j} \geq 0,~
    \sum_i m_{i \rightarrow j} = 1,~
    \sum_j m_{i \rightarrow j} = 1,~
\end{gathered}
\end{align}
where $d_{i, j}$ is the distance between $a_i$ and $b_j$ and
$m_{i \rightarrow j}$ is the mass moved from $a_i$ to $b_j$.
When $\{a_i\} \subset U$ and $\{b_j\} \subset f^{-1}(V)$,
it is \textbf{Wasserstein many-to-one}.
When $\{a_i\} \subset f(U)$ and $\{b_j\} \subset V$,
it is \textbf{Wasserstein discontinuity}.
High many-to-one likely implies low precision, and high discontinuity likely implies low recall.
The average of many-to-one and discontinuity is \textbf{Wasserstein cost}.

We note that our measures bypass some practical difficulties on using precision and recall as evaluation measures. 
The first issue was discussed in \cref{sec:simulation}, where we discussed that precision and recall are always equal when computed naively. 
This defeats their very purpose for capture both continuity and injectivity. 
Computing them based on \cref{eqn:dist_precision} and \cref{eqn:dist_recall} is more sensible, but it introduces another difficulty in practice due to high dimensionality: 
the radii $r_U$ and/or $r_V$ need to be quite large in order for some (outlier data point) $x$ to have a reasonable number of neighboring data points. 
Some $x$ ends up having many neighboring points, while others have very few\footnote{
This issue was also discussed in \citep{maaten2008visualizing}.}. 
This introduces a high variance on the number of neighboring data points across $x$. 
Our Wasserstein measures bypass both practical issues: having a fixed number of neighbors won't make $W_{2}(\mathbb{P}_{f(U)}, \mathbb{P}_V)$ and $W_{2}(\mathbb{P}_U, \mathbb{P}_{f^{-1}(V)})$ equal. 
In our experiments, we choose 30 neighboring points for all of $U$, $f^{-1}(V)$, $f(U)$ and $V$.

\subsection{Preliminary experiments on Wasserstein Measures, Compare Visualization Maps}
\vspace{\postsubsection}
\label{sec:additional_compare_maps}

\begin{figure*}[t]
\centering
\includegraphics[width=\linewidth]{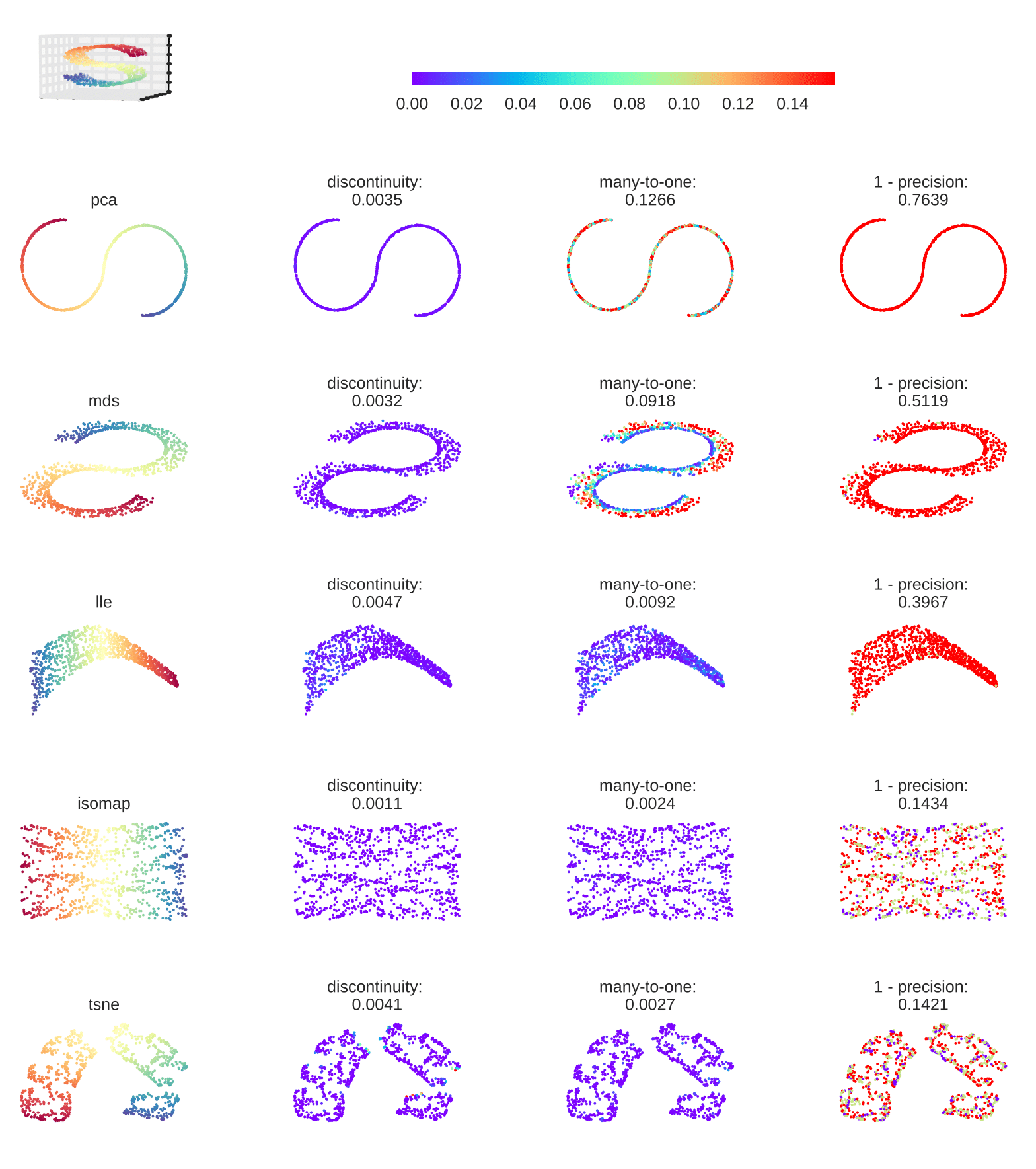}
\caption{quality of different methods on S-curve}
\label{fig:compare_methods_s_curve}
\end{figure*}

\begin{figure*}[t]
\centering
\includegraphics[width=\linewidth]{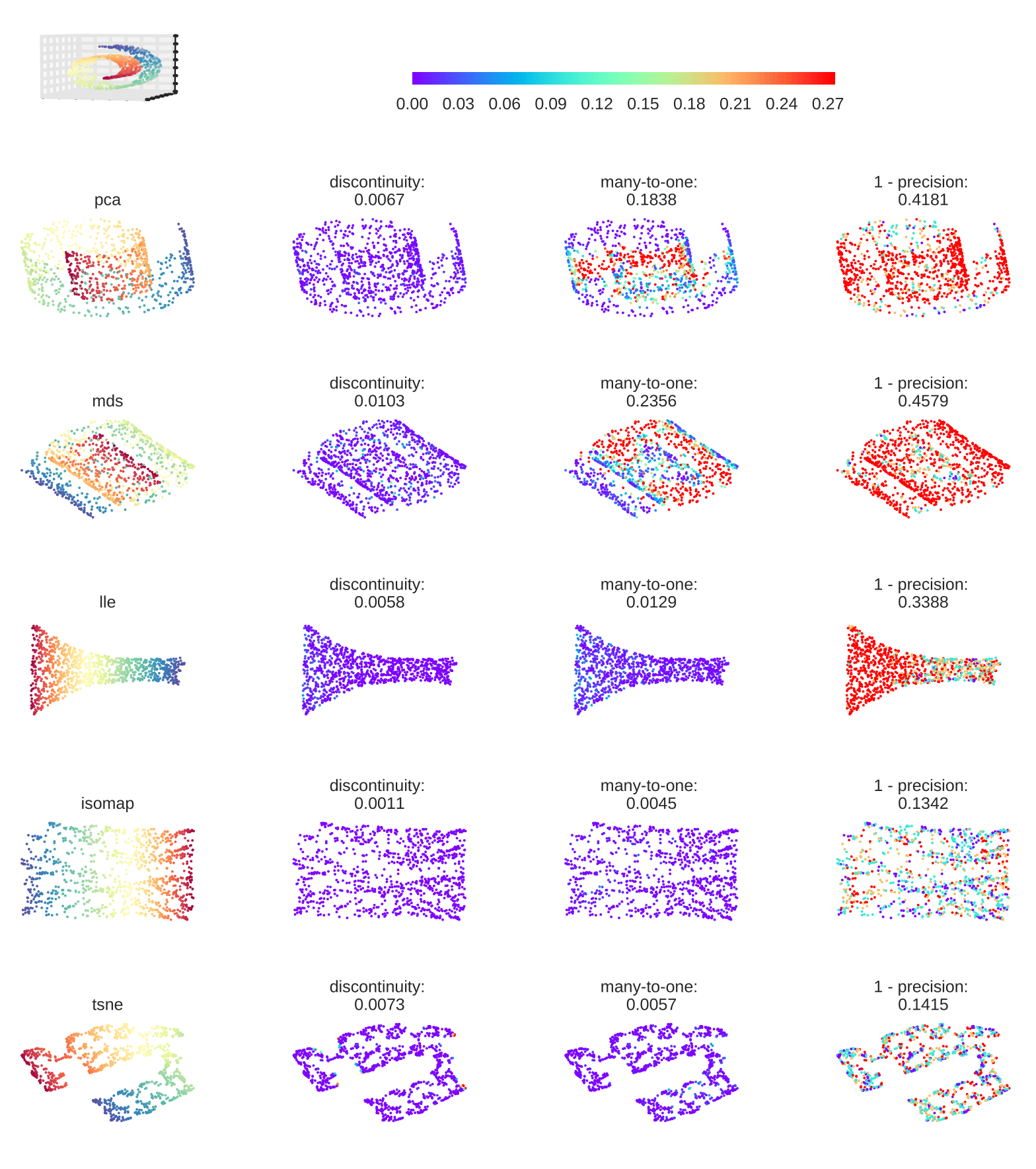}
\caption{quality of different methods on Swiss roll}
\label{fig:compare_methods_swiss_roll}
\end{figure*}

\begin{figure*}[t]
\centering
\includegraphics[width=\linewidth]{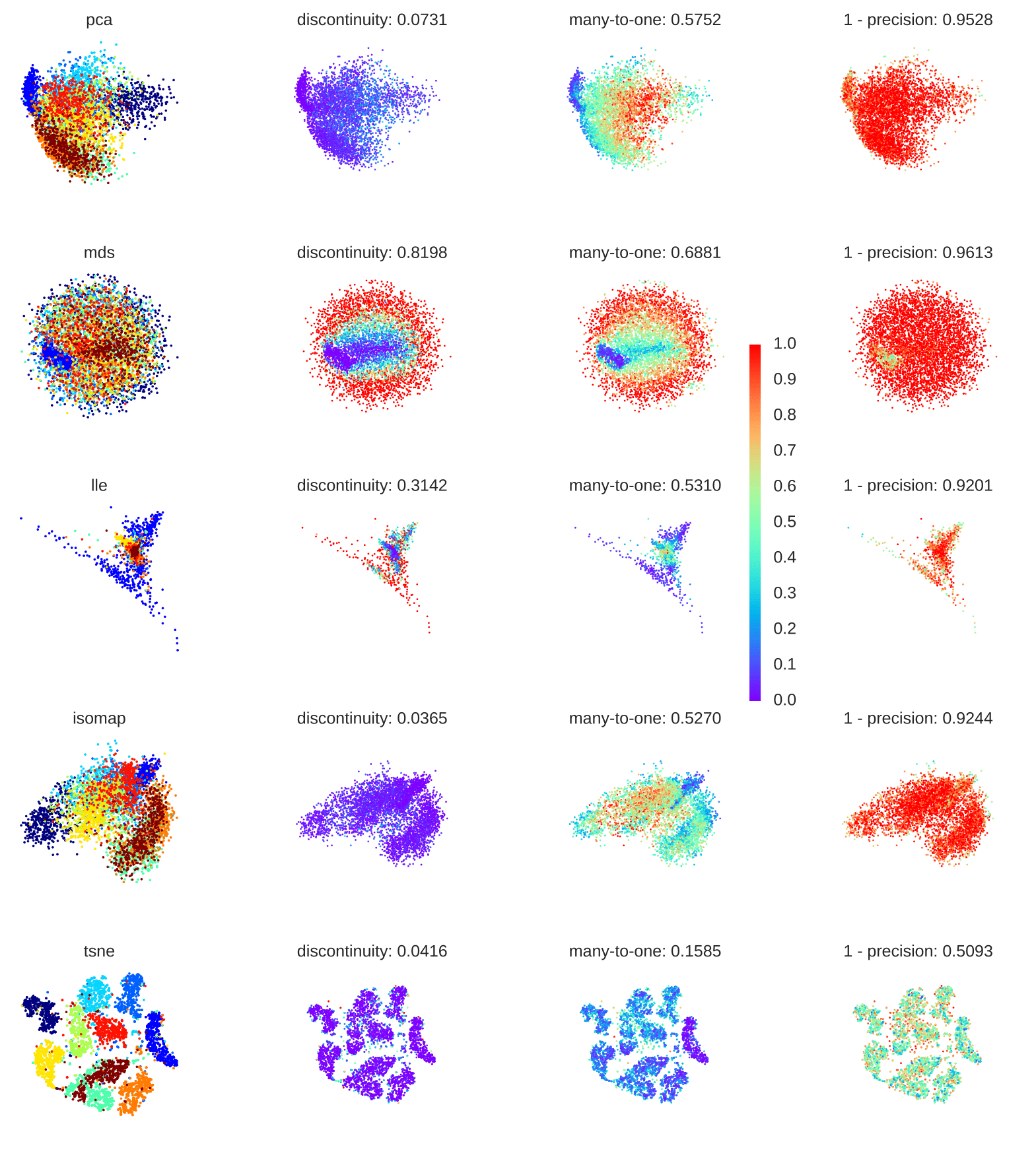}
\caption{quality of different methods on MNIST}
\label{fig:compare_methods_mnist}
\end{figure*}

\begin{figure*}[t]
\centering
\includegraphics[width=\linewidth]{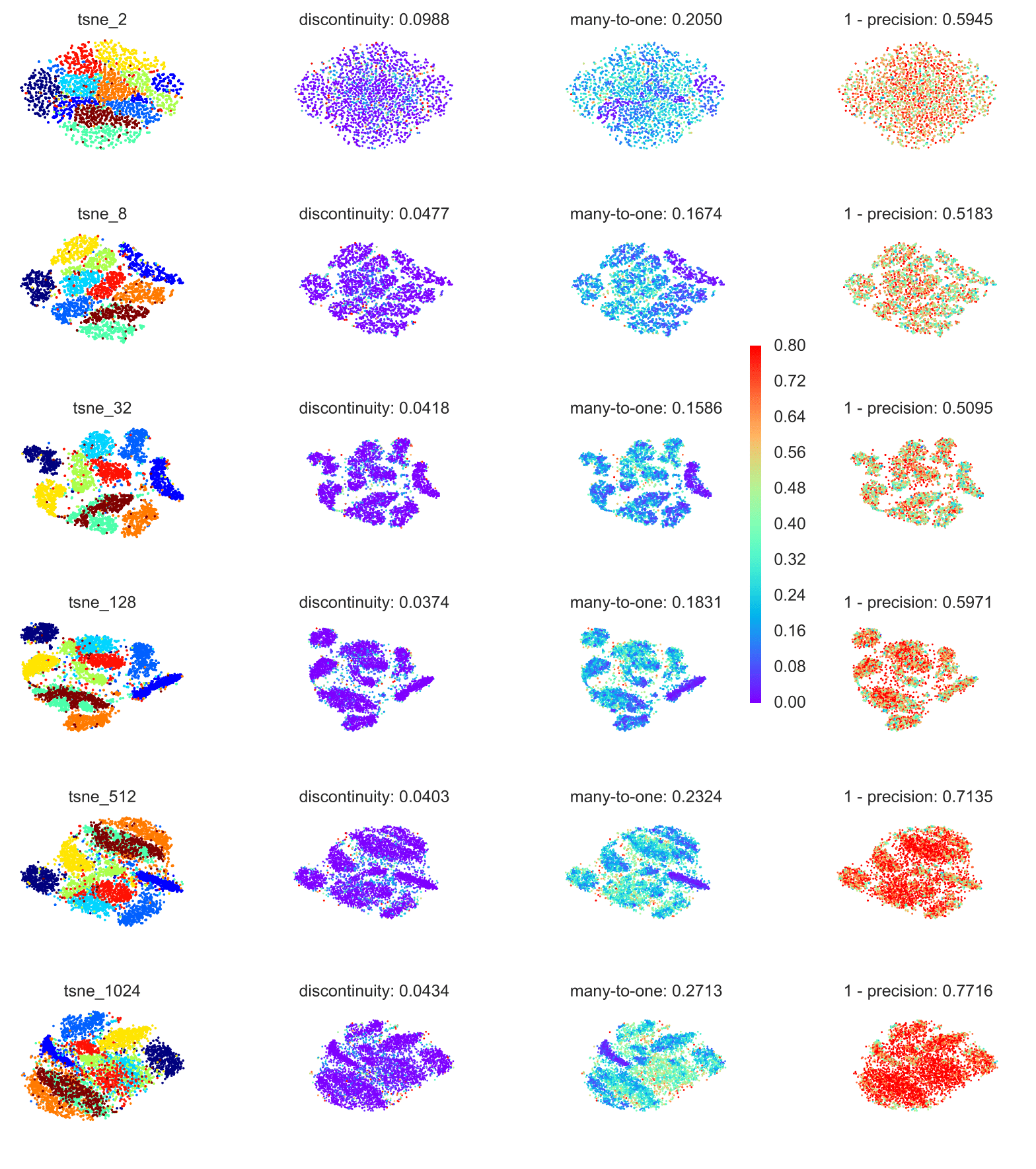}
\caption{quality of t-SNE with different perplexities on MNIST}
\label{fig:compare_perp_mnist}
\end{figure*}

In this section, 
we show preliminary results on using Wasserstein measures directly (instead of its lower bound) to choose between dimensionality reduction algorithms. 
We may interpret this as choosing between different information retrieval systems in the DR visualization context.  
Figure \ref{fig:compare_methods_s_curve} and
\ref{fig:compare_methods_swiss_roll} show the visualization results of
5 different methods on the S-curve and Swiss roll toy datasets respectively.
These include PCA, multidimensional scaling (MDS) \citeapp{borg2005modern}, locally linear embedding (LLE) \citeapp{roweis2000nonlinear}, Isomap \citeapp{tenenbaum2000global} and t-SNE \citeapp{maaten2008visualizing}.
In the results of PCA and MDS, the mappings squeeze the original data into
narrower regions in the 2D projection space. Squeezing naturally implies high degree of many-to-one.
At the same time, PCA mapping is linear, the MDS mapping in this case is close to linear,
which makes both PCA and MDS has a low discontinuity.
For S-curve and Swiss roll, LLE, Isomap and t-SNE all works well in the sense that they
successfully unwrapped the manifold. However, when local compression or
stretch happens, the Wasserstein discontinuity and many-to-one will
will increase slightly. For example, in the S-curve LLE results,
the right side of
data is compressed. Therefore it has a slightly higher many-to-one value,
while the discontinuity is still low.

Figure \ref{fig:compare_methods_mnist} shows the visualization results on MNIST digits.
As a linear map, PCA still has a relatively lower discontinuity and higher degree of many-to-one.
MDS preserve global distances, at the cost of sacrificing local distances. thus can map nearby points to far away locations, at the same time mapping far a way points together has poor local one-to-one property. So it has both high discontinuity and many-to-one on MNIST digits.
Compared with the previous toy example, LLE and Isomap both have a significant
performance drop.
Among all the methods, t-SNE still have the best local properties for MNIST digits, due to its neighborhood preservation objective.

\subsection{Preliminary experiments precision and recall (continuity v.s. injectivity) tradeoff}
\vspace{\postsubsection}
\label{sec:compare_perp}

\cref{thm:strong_precision_recall_tradeoff} suggests there is a trade-off between precision and recall, or equivalently continuity v.s. injectivity, via \cref{prp:equivalence_info_topology}. 
In this section, we attempt to illustrate this tradeoff phenomenon by altering the degree of continuity of a DR algorithm in a practical situation. 
We choose t-SNE on MNIST because: 
1) Heuristically t-SNE's perplexity parameter controls the degree of continuity: 
a higher perplexity means more neighboring data points will contract together and contraction is a continuous map (respectively, lower perplexity creates more tearing and spliting); 
2) the tradeoff may be best seen through DR algorithms that operate at the optimal tradeoff level. t-SNE has proved itself as the de facto standard for visualization in various datasets; 
3) As a practical dataset, MNIST visualization is still simple enough that humans can inspect and diagnose.

\cref{fig:compare_perp_mnist} shows visualizations with different t-SNE perplexity parameter. 
Each row is indexed by a different perplexity (perp $= 2, 8, \cdots, 1024$), with the intuition that the t-SNE DR map becomes more continuous with larger perplexity. 
The middle two columns are colored by our Wasserstein measures, with lower discontinuity costs representing more continuous maps (higher recall) and lower many-to-one costs indicating more injective (higher precision) maps.
The precision and recall tradeoff can be observed in the perplexity ranging from 32 to 128. As t-SNE becomes more continuous, it is also less injective. 
In this range, inspection by eye suggests t-SNE gives good visualizations.

Outside of the range of $(32, 128)$ both precision and recall become worse. 
We interpret this as t-SNE is giving relatively bad visualizations for these choices of parameter, as can be inspected by eye. 
For example, when perplexity $= 512$ and $1024$, t-SNE actually tends to have lower recall while precision worsens. 
When perplexity $< 32$, it is less clear whether it is due to: 
1) there is a tradeoff but our measures do not capture it. 
Our neighborhood size is also $30$ (comparable or bigger than the perplexity), so the scale may not be fair (on the other hand, choosing neighborhood size smaller than $30$ may introduce very high variance in the estimation); 
2) t-SNE actually performances worse on both continuity and injectivity, reflected by our measures. 
By inspection on the visualization, we believe it is probably because t-SNE isn't performing at any optimal level, so tradeoff cannot be seen.

\let\noopsort\undefined
\let\printfirst\undefined
\let\singleletter\undefined
\let\switchargs\undefined

\bibliographystyleapp{plainnat}
\bibliographyapp{nldr}

\end{document}